\documentclass{article}

\usepackage{arxiv}

\usepackage[utf8]{inputenc}
\usepackage[T1]{fontenc}
\usepackage{microtype}

\usepackage{amsmath}
\usepackage{amsfonts}
\usepackage{amssymb}
\usepackage{amsthm}
\usepackage{bbm}
\usepackage{nicefrac}

\newtheorem{proposition}{Proposition}
\newtheorem{corollary}[proposition]{Corollary}

\usepackage{natbib}
\usepackage{hyperref}
\usepackage{url}

\usepackage{booktabs}
\usepackage{multirow}
\usepackage{makecell}
\usepackage{array}
\usepackage{colortbl}

\usepackage{graphicx}
\usepackage{wrapfig}
\usepackage{tikz}
\usetikzlibrary{arrows.meta, positioning}

\usepackage{xcolor}
\usepackage{enumitem}
\usepackage{etoolbox}
\usepackage{caption}

\AtBeginEnvironment{wrapfigure}{\captionsetup{font=footnotesize}}
\AtBeginEnvironment{wraptable}{\captionsetup{font=footnotesize}}
\usepackage{float}

\AtBeginEnvironment{center}{\vspace{-4pt}}
\AfterEndEnvironment{center}{\vspace{-4pt}}

\newcommand{\BAEE}{\textsc{baee}}
\newcommand{\EFA}{\textsc{efa}}
\newcommand{\ATLT}{\textsc{atlt}}
\newcommand{\ED}{\textsc{ed}}
\newcommand{\PSC}{\textsc{psc}}

\setlist[itemize]  {leftmargin=1.35em, topsep=2pt, itemsep=1.5pt, parsep=0pt, partopsep=0pt}
\setlist[enumerate]{leftmargin=1.50em, topsep=2pt, itemsep=1.5pt, parsep=0pt, partopsep=0pt}
\setlist[description]{leftmargin=1.35em, topsep=2pt, itemsep=1.5pt, parsep=0pt, partopsep=0pt}

\title{The Detection--Extraction Gap:\\Models Know the Answer Before They Can Say It}

\author{%
  Hanyang Wang\thanks{Corresponding author (\texttt{hanyangw@uchicago.edu}).}\\
  The University of Chicago
  \And
  Mingxuan Zhu\\
  Imperial College London\\
  \texttt{mz3225@ic.ac.uk}
}

\begin{document}

\maketitle

\begin{abstract}
Modern reasoning models continue generating long after the answer is already determined. 
Across five model configurations, two families, and three benchmarks, we find that \textbf{52--88\% of chain-of-thought tokens are produced after the answer is recoverable} from a partial prefix. 
This post-commitment generation reveals a structural phenomenon: the \textbf{detection--extraction gap}. 
Free continuations from early prefixes recover the correct answer even at 10\% of the trace, while forced extraction fails on 42\% of these cases. 
The answer is recoverable from the model state, yet prompt-conditioned decoding fails to extract it. 
We formalize this mismatch via a total-variation bound between free and forced continuation distributions, yielding quantitative estimates of suffix-induced shift. 
Exploiting this asymmetry, we propose Black-box Adaptive Early Exit (\BAEE{}), which uses free continuations for both detection and extraction, truncating \textbf{70--78\% of serial generation} while \textbf{improving accuracy by 1--5\,pp} across all models.
For thinking-mode models, early exit prevents post-commitment overwriting, yielding gains of up to 5.8\,pp; a cost-optimized variant achieves 68--73\% reduction at a median of 9 API calls.
Code is available at \url{https://github.com/EdWangLoDaSc/know2say}.
  \end{abstract}

\section{Introduction}
\label{sec:intro}

Modern reasoning models often generate thousands of tokens after the
answer is already determined; this behavior is widespread, not a rare edge case.
The majority of chain-of-thought (CoT) tokens are generated after the
correct answer is already robustly recoverable from a partial prefix
(Figure~\ref{fig:hero}a).
Prior work has established this via internal
probes~\citep{boppana2026reasoning,zhang2025reasoning,cox2026decoding},
but those approaches require open-weight access.
Using only black-box API access~(\S\ref{sec:method}), we confirm the
phenomenon and reveal a surprising structure beneath it.

\begin{wrapfigure}{r}{0.61\linewidth}
\vspace{-0.6\baselineskip}
\centering
\includegraphics[width=\linewidth]{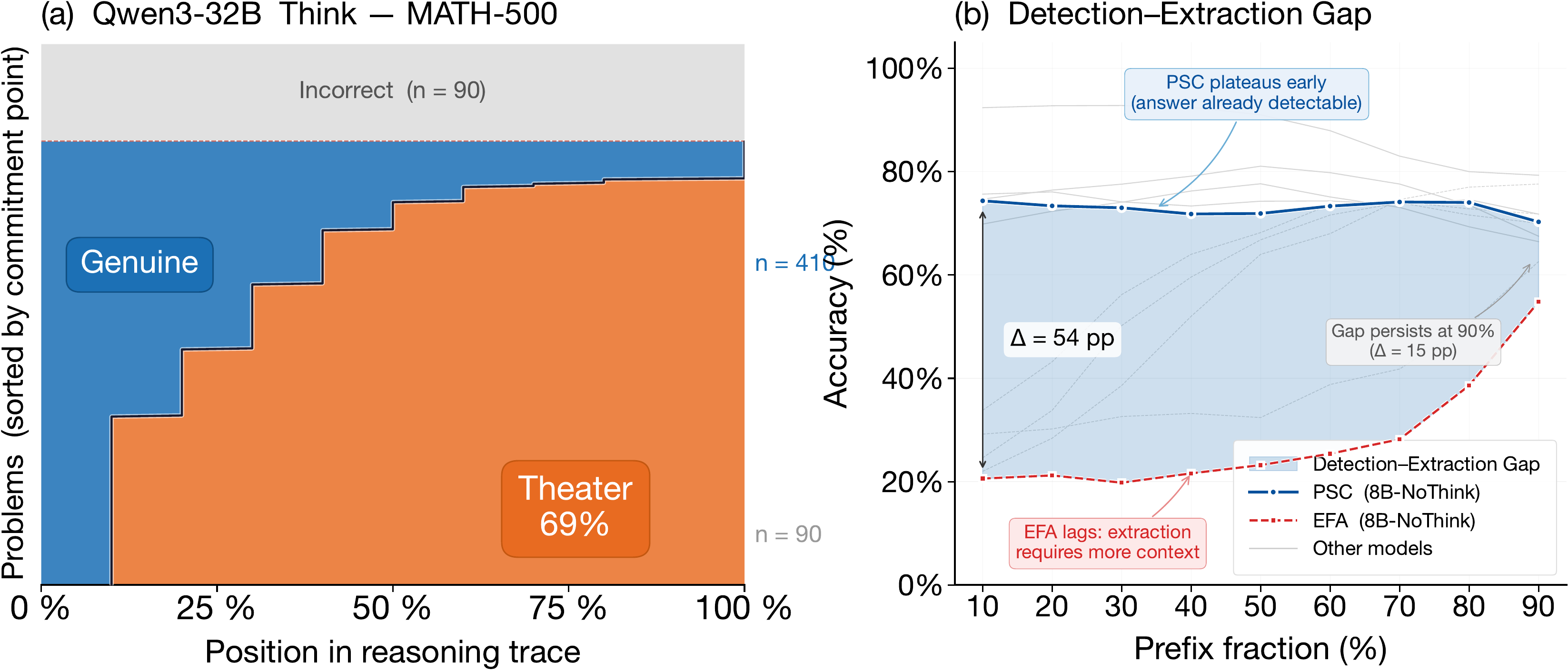}
\caption{\textbf{Two core findings.}
  (a)~Commitment map for 32B-Think: 69\% of CoT tokens follow the commitment
  boundary (75\% under PSC threshold; Table~\ref{tab:main}).
  (b)~The detection--extraction gap: \PSC{} detects recoverability from 10\%
  prefix, yet \EFA{} fails to extract on 42\% of these problems.
  \BAEE{} exploits this structure for 70--78\% serial reduction with 1--5\,pp accuracy gains (\S\ref{sec:baee}).}
\label{fig:hero}
\vspace{-1\baselineskip}
\end{wrapfigure}

Our central finding is the \textbf{detection--extraction gap}, a
structural mismatch between what the model's state \emph{encodes} and what
constrained decoding can \emph{elicit}.
The model ``knows'' the answer but cannot ``say'' it when forced:
free continuations recover the answer, while forced extraction fails on
nearly half of the same problems.
We relate measured gaps to suffix-induced shift through a TV inequality
(Proposition~\ref{prop:tv_bound}; \S\ref{sec:gap}) and
confirm three predictions across models, benchmarks, and suffix variants.

The gap has immediate practical consequences: any early-exit strategy that
forces extraction inherits this failure mode, dropping accuracy by
41--62\,pp (\S\ref{sec:baee}).
\PSC{}-triggered \textbf{Black-box Adaptive Early Exit} (\BAEE{}) uses free
continuations for both detection and extraction, reducing 70--78\% of serial
generation while \emph{improving} accuracy by 1--5\,pp on all models,
with up to $+5.8$\,pp gains on thinking models via overthinking prevention.
Using four black-box probes (\PSC{}, \EFA{}, \ATLT{}, \ED{};
\S\ref{sec:method}) that require only sampling or logprob API access,
our contributions are:
\begin{enumerate}
  \item \textbf{The detection--extraction gap}: behavioral recoverability
        (\PSC{}) and forced extractability (\EFA{}) diverge structurally on
        partial traces; we formalize the gap via a lower bound on
        suffix-induced distributional shift
        (Proposition~\ref{prop:tv_bound}; Appendix~\ref{app:gap_mechanism})
        and confirm three quantitative predictions across suffix variants and
        prefix lengths (\S\ref{sec:gap}; Appendix~\ref{app:suffix_ablation}).
        Across two model families, three benchmarks, and five configurations,
        the gap behaves as a \emph{structural} property rather than
        model-specific artifact.
  \item \textbf{Task-topology-dependent commitment structure}: \PSC{}
        trajectories are non-monotone on closed-form tasks (MATH-500) but
        monotone on sustained reasoning tasks (GPQA-Diamond); the
        Think--NoThink commitment gap collapses on code generation (HumanEval).
        These patterns show how task structure governs when and how answers
        become recoverable (\S\ref{sec:generalization};
        Appendices~\ref{app:psc_mono},~\ref{app:humaneval}).
  \item \textbf{\BAEE{}}: an early-exit policy that exploits the gap by using
        free continuations for both detection and extraction, achieving
        70--78\% serial reduction while \emph{improving} accuracy by 1--5\,pp;
        for thinking-mode models, gains reach 5.8\,pp by interrupting
        post-commitment overthinking (\S\ref{sec:baee}).
\end{enumerate}


\vspace{-0.5em}

\section{Related Work}
\label{sec:related}

We organize related work into three categories.
\textbf{White-box probing} requires internal model access:
\citet{boppana2026reasoning} detect early commitment via hidden-state probes,
\citet{zhang2025reasoning} find self-verification signals in residual streams,
and \citet{cox2026decoding} show answers are encoded before any CoT begins;
these methods are limited to open-weight models and cannot observe the
detection--extraction gap, which is only visible through generation behavior.
\textbf{CoT faithfulness and self-consistency} provide the conceptual
foundation: reasoning traces need not faithfully reflect internal
decisions~\citep{turpin2023language,barez2025chain,jiang2025misaligning,dettki2025reasoning},
a decoupling that maps directly onto our gap.
However, this literature studies faithfulness at the \emph{problem} level
(does the trace reflect the decision?) without probing \emph{when} during
generation the answer becomes recoverable.
\PSC{} extends self-consistency decoding~\citep{wang2022self} and
agreement-as-confidence methods~\citep{xiong2023can,rivera2024combining,sun2024confidence}
by producing a recoverability \emph{trajectory} across prefix fractions,
enabling the gap analysis that prior SC methods cannot perform.
\textbf{Inference-time early exit and overthinking} covers
layer-level exits~\citep{chen2023ee,elhoushi2024layerskip},
adaptive test-time compute~\citep{snell2024scaling}, and sequence-level
truncation~\citep{sprague2024cot,cuadron2025danger,sui2025stop,guan2025monitoring,zhang2025ascot};
\citet{yang2025dynamic} monitor token-level confidence (semi-white-box) and
\citet{wang2025chain} probe step necessity.
These methods optimize \emph{when to exit} but do not ask \emph{why na\"ive
extraction fails}: none can observe the detection--extraction gap, which
emerges only when comparing free-continuation and forced-extraction probes on
the same prefix.
This gap is our central structural contribution; \BAEE{} is its
practical consequence (Pareto comparison with DEER in
Appendix~\ref{app:deer_comparison}).

\vspace{-0.5em}
\section{Method}
\label{sec:method}
\vspace{-0.5em}
Given a problem $p$ and a full rollout $y_{1:T}$ ($T$ tokens), we define a
checkpoint grid $F = \{f_1,\ldots,f_m\} \subset (0,1)$ with prefix length
$k_f = \lfloor f \cdot T \rfloor$.
At each checkpoint, we probe the prefix using two core protocols (sampling-only)
and two auxiliary diagnostics (logprob-based).
Figure~\ref{fig:framework} illustrates the complete pipeline.

\begin{figure}[ht]
    \centering
    \includegraphics[width=1\textwidth]{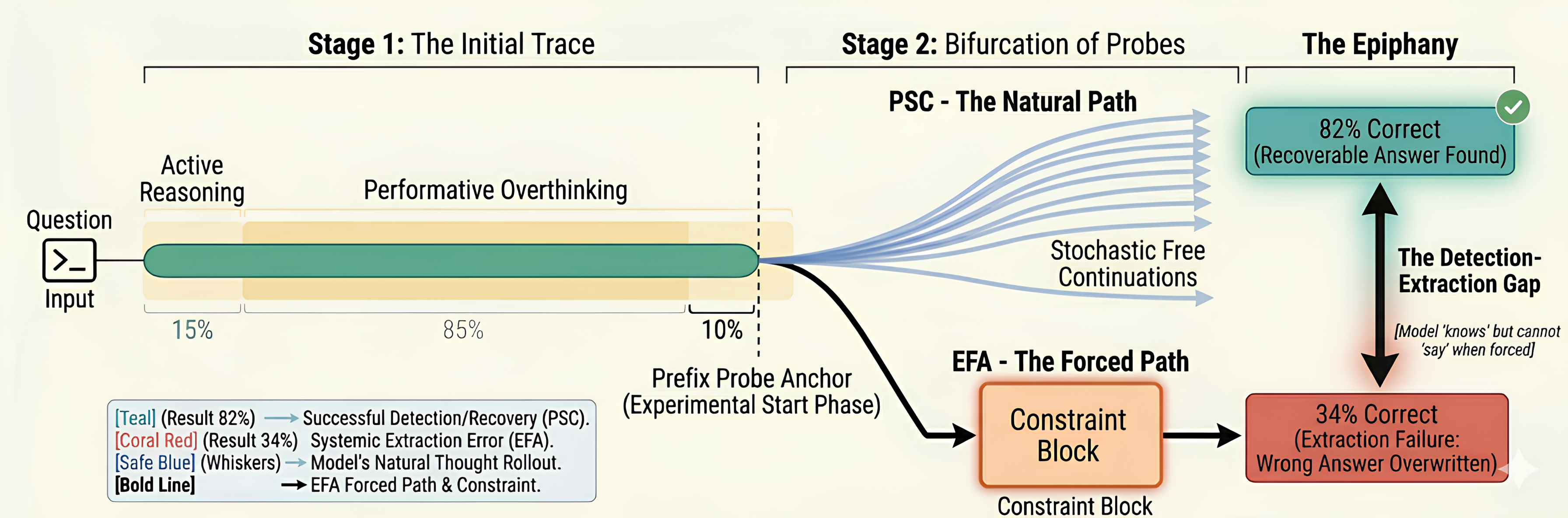}
    \caption{\textbf{Method pipeline.}
    At prefix fraction $f{=}0.1$, \textbf{\PSC{}} (free continuation) recovers
    the correct answer 82\% of the time, while \textbf{\EFA{}} (forced extraction)
    succeeds on only 34\%.
    \BAEE{} uses free continuations for both detection and extraction,
    sidestepping the gap entirely.}
    \label{fig:framework}
\end{figure}

\subsection{Core Protocols}
\label{sec:core_protocols}

\paragraph{Early Forced Answering (\EFA{}).}
\label{sec:efa}
\EFA{} tests whether the correct answer can be \emph{forced out} of a prefix by appending
an answer-inducing suffix:
\begin{equation}
  \EFA{}(p,\,c_{1:k}) = \texttt{greedy\_decode}\!\left(
    p \oplus c_{1:k} \oplus
    \texttt{``\textbackslash nTherefore, the final answer is \textbackslash boxed\{''}\right)
\end{equation}
Decoding stops at the first \texttt{\}} or after 64 tokens.
\EFA{} measures \emph{forced extractability}: can the model produce the correct
answer when explicitly asked?

\paragraph{Prefix Self-Consistency (\PSC{}).}
\label{sec:psc}
\PSC{} tests whether the answer is \emph{naturally recoverable} by sampling
$N=8$ independent continuations from prefix $c_{1:k}$:
\vspace{-0.5em}
\begin{equation}
  \PSC{}(p,\,c_{1:k}) =
    \frac{1}{N}\sum_{i=1}^{N}
    \mathbbm{1}\!\left[\text{correct}\!\left(\text{sample}_i(p \oplus c_{1:k})\right)\right]
\end{equation}

\vspace{-0.5em}
\PSC{} operationalizes \emph{recoverability} as a Monte Carlo estimator of
a well-defined distributional quantity:
\begin{proposition}[PSC concentration]
\label{prop:psc_estimator}
Let $p_k := P_{\mathrm{free}}(a^* \mid c_{1:k})$.
Since the $N$ samples are i.i.d., $\mathrm{PSC}_N(k)$ is unbiased for $p_k$ with
$P(|\mathrm{PSC}_N(k) - p_k| \geq \varepsilon) \leq 2e^{-2N\varepsilon^2}$ (Hoeffding).
At $N{=}8$: within $\pm 0.25$ of $p_k$ with probability $\geq 91\%$.
\end{proposition}
\vspace{-0.3em}
\noindent Unlike \EFA{}, \PSC{} injects no forced context.
Unlike standard SC~\citep{wang2022self}, it produces a \emph{trajectory}
$\{\PSC{}(f)\}_{f \in F}$ revealing \emph{when} commitment emerges, enabling
the detection--extraction gap analysis (\S\ref{sec:gap}).

\smallskip\noindent The conceptual distinction is critical:
\EFA{} and \PSC{} both operate on the same prefix but measure fundamentally different
properties.
A prefix can exhibit high \PSC{} (answer recoverable via free continuation)
yet low \EFA{} accuracy (forced extraction fails).
This asymmetry, the \textbf{detection--extraction gap}, is our central finding.

\subsection{Commitment, Gap, and Early Exit}
\label{sec:defs}

The \textbf{commitment fraction} (a behavioral recoverability metric, not a
claim about internal states) is defined as the earliest checkpoint at which \PSC{}
crosses threshold $\theta$:
\begin{equation}
  f^*_\theta = \min\!\left\{f \in F : \PSC{}\!\left(p,\, c_{1:k_f}\right) \geq \theta\right\}
\end{equation}
The \textbf{post-commitment fraction} $1 - f^*_\theta$ is the share of tokens
generated after the answer becomes recoverable.
The \textbf{detection--extraction gap},
$\Delta_f = \text{PSC trigger rate}_f - \text{EFA accuracy}_f$,
quantifies the discrepancy between the two core probes at each checkpoint.
We use $\theta = 0.75$ (Think/GPT-OSS) and $\theta = 0.875$ (NoThink).

\paragraph{BAEE: Black-box Adaptive Early Exit.}
\label{sec:baee_def}
\BAEE{} operationalizes behavioral recoverability as an early-exit policy:
check \PSC{} starting from $f=0.10$; if $\PSC{} \geq \theta$, exit and
return the majority answer from the continuations.
Since 67--92\% of problems trigger at the first checkpoint, the median cost is
$N+1=9$ API calls.
By using free continuations for both \emph{detection} and \emph{extraction},
\BAEE{} avoids the detection--extraction gap entirely (\S\ref{sec:baee}).

\subsection{Auxiliary Diagnostics}
\label{sec:auxiliary}

Two logprob-based protocols provide independent corroboration (neither is used in \BAEE{}):

\paragraph{Answer Token Logprob Trajectory (\ATLT{}).}
\label{sec:atlt}
Mean log-probability of the correct-answer tokens after prefix $c_{1:k}$:
$\ATLT{} = \frac{1}{|a_\text{tok}|}\sum_t \log P(a_\text{tok}^{(t)} \mid
p \oplus c_{1:k} \oplus a_\text{tok}^{(1:t-1)})$.

\vspace{-0.4em}
\paragraph{Entropy Dynamics (\ED{}).}
\label{sec:ed}
Per-token entropy via top-$k$ logprobs ($k=20$):
$H_t = -\sum_{i=1}^{k} p_i \log p_i - p_\text{tail} \log p_\text{tail}$.
Post-commitment entropy patterns distinguish Think from NoThink models
(\S\ref{sec:entropy}).

\subsection{Experimental Setup}
\label{sec:setup}

Table~\ref{tab:setup} consolidates models, benchmarks, protocol settings, and
the re-grading correction used throughout this paper.

\begin{table}[ht]
\centering
\caption{Compact summary of the experimental setup.}
\label{tab:setup}
\small
\setlength{\tabcolsep}{8pt}
\renewcommand{\arraystretch}{1.2}
\begin{tabular}{@{} l p{0.75\linewidth} @{}}
\toprule
\textbf{Category} & \textbf{Details} \\
\midrule
\textbf{Models}     & Qwen3-32B-(Think/NoThink), Qwen3-8B-(Think/NoThink)~\cite{yang2025qwen3}, and GPT-OSS-120B (medium reasoning)~\cite{agarwal2025gpt}. \\
\addlinespace[0.5em]
\textbf{Benchmarks} & \textbf{MATH-500}: 500 problems (primary); \textbf{GPQA-Diamond}: 198 problems (validation); \textbf{HumanEval}: 164 problems (code). \\
\addlinespace[0.5em]
\textbf{Prefix Grid} & $\{0.10, 0.20, \dots, 0.90\}$ (9 checkpoints). \\
\addlinespace[0.5em]
\textbf{Sampling}    & Full rollouts: 4 (temp.\ 1.0); \PSC{} samples: 8; entropy top-$k$: 20; \EFA{} max tokens: 64. \\
\addlinespace[0.5em]
\textbf{Grading}     & SymPy symbolic equivalence (MATH-500); Exact match (GPQA). \\
\addlinespace[0.5em]
\textbf{Re-grading}  & Fixed \EFA{} suffix-stripping bug (90 evaluations flipped). Verified against independent SymPy re-grade with zero discrepancies. \\
\bottomrule
\end{tabular}
\end{table}

\vspace{-0.5em}

\section{Results}
\label{sec:results}
\vspace{-0.5em}
Early commitment (\S\ref{sec:early_commitment}) establishes the premise;
the detection--extraction gap (\S\ref{sec:gap}) is the central finding;
generalization (\S\ref{sec:generalization}) and robustness
(\S\ref{sec:robustness}) validate it; \BAEE{} (\S\ref{sec:baee}) exploits it.

\vspace{-0.5em}
\subsection{Early Behavioral Commitment}
\label{sec:early_commitment}
\vspace{-0.5em}

Table~\ref{tab:main} and Figure~\ref{fig:main} show that commitment occurs before 50\%
of the CoT across configurations (25\% for 32B-Think to 48\% for 8B-NoThink), and \PSC{} accuracy
at the first checkpoint (10\%) already reaches 82--96\%.
Commitment scales with difficulty: 13\% (Level~1) to 63\% (Level~5;
Appendix~\ref{app:difficulty_aime}).
The key question is whether this early recoverability can be
\emph{exploited}, which requires understanding why na\"ive extraction fails.

\vspace{-1em}
\begin{table}[ht]
\centering
\caption{Post-commitment generation metrics on MATH-500 (re-graded).
         ``Acc.''\ = fraction of problems where the \emph{first} of 4 sampled rollouts
         is correct (single-rollout accuracy, not best-of-4).
         Commitment fraction defined conditionally on solvable instances
         ($\geq$1/4 correct); CIs from 10K bootstrap resamples.
         Latency reduction shown for \PSC{}-8 all (accuracy-optimal; see Table~\ref{tab:baee} for all strategies).}
\label{tab:main}
\footnotesize
\setlength{\tabcolsep}{4.5pt}
\renewcommand{\arraystretch}{1.25}
\begin{tabular}{@{} l c ccc ccc @{}}
\toprule
\multirow{2}{*}{\textbf{Model}} & \multirow{2}{*}{\textbf{Acc.}}
  & \multicolumn{3}{c}{\textbf{Commitment profile}}
  & \multicolumn{3}{c}{\textbf{Latency reduction}} \\
\cmidrule(lr){3-5} \cmidrule(lr){6-8}
  & & \textbf{Commit} & \textbf{95\% CI} & \textbf{Post-commit}
  & \textbf{CoT len} & \textbf{Main-rollout red.} & \textbf{Avg calls} \\
\midrule
\rowcolor{blue!7}
Qwen3-32B-Think   & 82\% & 25\% & [23\%, 26\%] & \textbf{75\%} & 2878 & 72\% & 7.7 \\
\rowcolor{orange!8}
Qwen3-32B-NoThink & 91\% & 39\% & [35\%, 42\%] & 61\% &  698 & 73\% & 8.0 \\
\rowcolor{blue!7}
Qwen3-8B-Think    & 76\% & 27\% & [25\%, 29\%] & \textbf{73\%} & 2881 & 68\% & 7.3 \\
\rowcolor{orange!8}
Qwen3-8B-NoThink  & 90\% & 48\% & [44\%, 52\%] & 52\% &  696 & 73\% & 8.0 \\
\rowcolor{gray!6}
GPT-OSS-120B      & 96\% & 36\% & [34\%, 38\%] & 64\% &  823 & \textbf{85\%} & 8.6 \\
\bottomrule
\end{tabular}
\end{table}

\begin{figure}[ht]
\centering
\includegraphics[width=\linewidth]{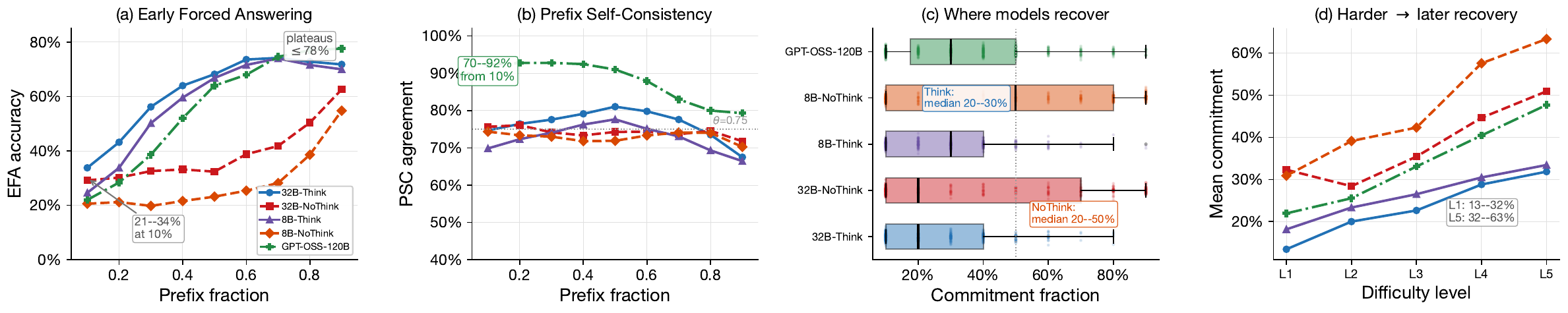}
\caption{Main results.
  (a)~\EFA{} accuracy by prefix length.
  (b)~\PSC{} agreement ($>$70\% from the first checkpoint).
  (c)~Commitment distributions (Think median $\sim$25\%, NoThink $\sim$40\%).
  (d)~Commitment increases monotonically with problem difficulty.}
\label{fig:main}
\end{figure}

\vspace{-1em}
\subsection{The Detection--Extraction Gap}
\label{sec:gap}

\paragraph{The phenomenon.}
``Recoverable'' and ``extractable'' are not the same.
At the 10\% prefix for 32B-Think: \textbf{70\% of problems} are
recoverable via free continuation (\PSC{} $\geq 75\%$), yet only
\textbf{34\%} yield a correct forced answer (\EFA{}).
\EFA{} fails on the majority of recoverable problems, and this holds
across five different forcing suffixes (39--94\,pp gaps;
Appendix~\ref{app:suffix_ablation}), ruling out format sensitivity.

\paragraph{Case studies.}
Two examples from 32B-Think illustrate distinct failure modes of \EFA{} at
10\% prefix, both with \PSC{} = 8/8 (Table~\ref{tab:case_study}):

\begin{table}[t]
\centering
\caption{\textbf{Detection--extraction gap: two case studies} (Qwen3-32B-Think,
  10\% prefix, \PSC{} 8/8). \EFA{} (shaded) extracts the wrong answer despite
  full recoverability via free continuation.}
\label{tab:case_study}
\footnotesize
\setlength{\tabcolsep}{7pt}
\renewcommand{\arraystretch}{1.45}
\begin{tabular}{@{} p{4.6cm}
                    c
                    >{\centering\arraybackslash}c
                    >{\centering\columncolor{red!7}\arraybackslash}c
                    l @{}}
\toprule
\textbf{Problem} & \textbf{GT}
  & \textbf{\textcolor{teal!80!black}{PSC@10\%}}
  & \textbf{\textcolor{red!65!black}{EFA@10\%}}
  & \textbf{Failure mode} \\
\midrule
$1 - 2 + 3 - 4 + \cdots + 99 - 100$
  & $-50$ & \textcolor{teal!80!black}{8/8} & $50$ & Sign dropped \\
$f(x)=2^x$; find $\sqrt{f(f(f(f(1))))}$
  & $256$ & \textcolor{teal!80!black}{8/8} & $4$ & Intermediate value \\
\bottomrule
\end{tabular}
\end{table}

\begin{wrapfigure}{r}{0.52\linewidth}
\vspace{-0.5\baselineskip}
\centering
\includegraphics[width=\linewidth]{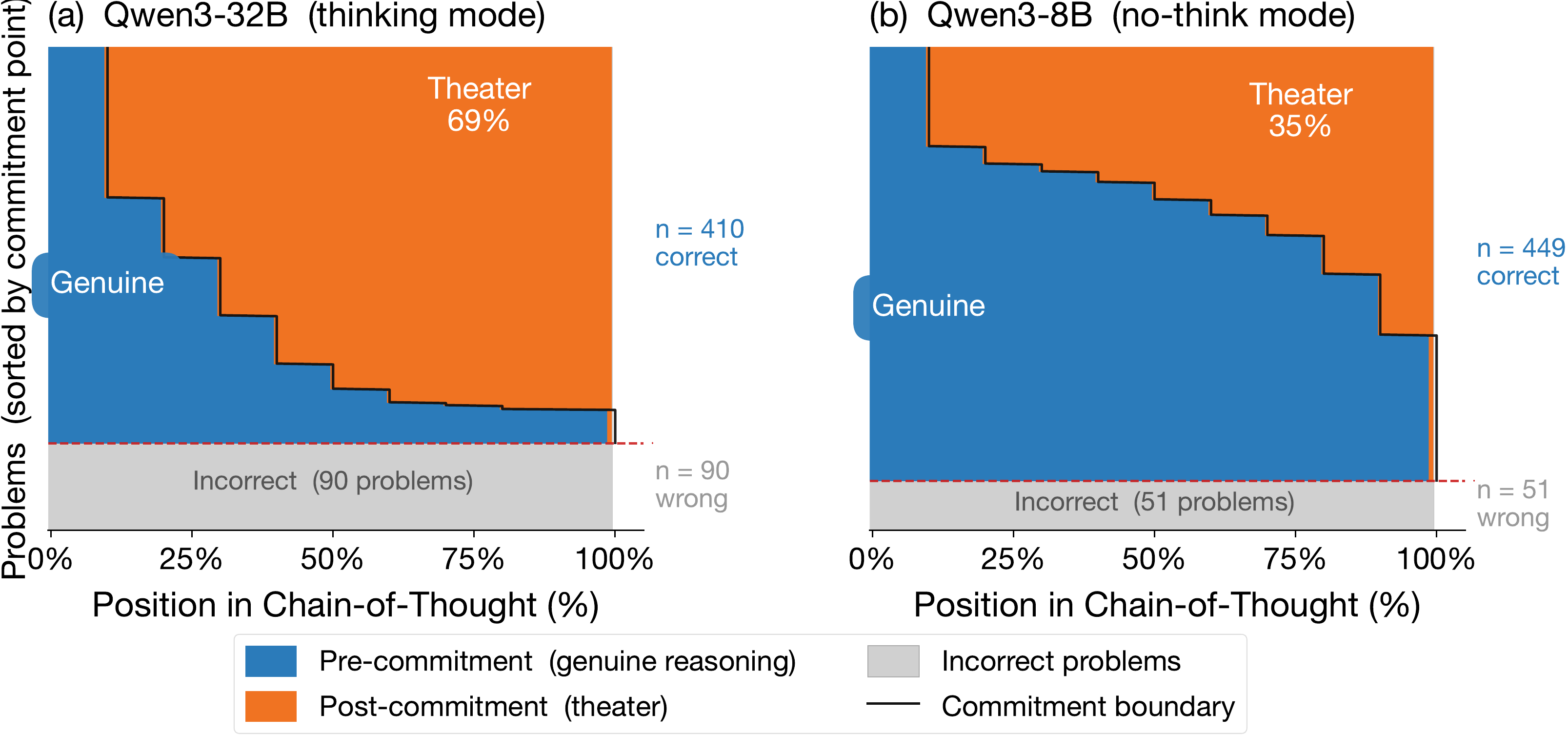}
\caption{Commitment maps.
  (a)~32B-Think: 69\% post-commitment.
  (b)~8B-NoThink: 35\% post-commitment, later and more variable.}
\label{fig:theater_map}
\vspace{-1\baselineskip}
\end{wrapfigure}

\noindent \textbf{Sign drop}: \EFA{} returns $50$ (correct magnitude, wrong sign),
persisting through 80\% of the CoT.
\textbf{Intermediate value}: \EFA{} returns $4 = f(1)$, the result of the first
step of the nested evaluation rather than the final answer.
In both cases, the forcing suffix acts as a \emph{distribution shift},
eliciting premature outputs from a partially-formed reasoning state that
free continuation resolves correctly.

\vspace{-0.5em}
\paragraph{Systematic characterization.}
Among the 208 gap instances (32B-Think, $f=0.10$): 59\% are short random outputs
($\leq$2 characters), 11.5\% are near-misses, and 0\% are format-only errors.
\PSC{} level does not predict \EFA{} outcome (98.4\% vs 97.5\%), confirming that
detection and extraction probe different properties.
A full failure taxonomy is in Appendix~\ref{app:gap_mechanism}.
The gap is most extreme for GPT-OSS-120B ($\Delta_{0.1} = 70.4$\,pp:
\PSC{}@10\% = 92.4\%, \EFA{}@10\% = 22.0\%), ruling out difficulty as a confound.

\paragraph{From failure modes to mechanism.}
The observed failures are not random: models systematically emit locally
coherent but globally incomplete states (correct partial results, sign
errors that preserve magnitude, premature terminations).
Combined with the suffix-invariance established above, these patterns
point to a mismatch between \emph{reachability} and \emph{constrained
decoding}.
\PSC{} succeeds as the correct answer lies in a high-probability region of
the continuation distribution, whereas \EFA{} fails due to suffix-induced
distributional shift.
Latent answer information thus emerges well before it can be reliably
\emph{externalized} under constrained decoding, a temporal gap between
knowledge encoding and policy readiness
(Appendix~\ref{app:gap_mechanism}).

\paragraph{Distributional-shift framework.}
The gap is governed by suffix-induced distributional shift.
Let $P_{\mathrm{free}}(\cdot \mid c_{1:k})$ and $P_{\mathrm{forced}}(\cdot \mid c_{1:k},s)$
be the free and forced continuation distributions.
\begin{proposition}[Gap--TV bound]
\label{prop:tv_bound}
$\mathrm{gap}_k := \mathrm{PSC}(k) - \mathrm{EFA}(k) \;\leq\;
d_{\mathrm{TV}}(P_{\mathrm{free}},\, P_{\mathrm{forced}})$.
\end{proposition}
\vspace{-0.5em}
\begin{proof}
$d_{\mathrm{TV}}(P,Q) \!=\! \sup_A|P(A)\!-\!Q(A)| \geq |P(\{a^*\})\!-\!Q(\{a^*\})| = \mathrm{gap}_k$.\qedhere
\end{proof}
\vspace{-0.3em}
\noindent Combined with Proposition~\ref{prop:psc_estimator}, the measured gap
is a consistent estimator of $p_k - P_{\mathrm{forced}}(a^*)$, providing
lower bounds on $d_{\mathrm{TV}}$
(Table~\ref{tab:tv_bounds}; Appendix~\ref{app:gap_mechanism}).
The framework yields three \emph{a priori} predictions, all confirmed:
\begin{enumerate}
  \item \textbf{Gap decreases with $f$}: as $f \to 1$, the suffix
    becomes a natural continuation and the shift vanishes.
    Confirmed: 54\,pp at $f{=}0.10$ to $<$5\,pp at $f{=}0.70$
    (Figure~\ref{fig:gap}b).
  \item \textbf{Larger shifts yield larger gaps}: non-\texttt{\textbackslash
    boxed} suffixes impose a stronger shift and produce 56--94\,pp gaps
    vs.\ 39--45\,pp for \texttt{\textbackslash boxed}
    (Appendix~\ref{app:suffix_ablation}).
  \item \textbf{More structured intermediate states amplify the gap}:
    Think models (richer in-flight computation) show larger early-prefix
    gaps than NoThink when normalized by \PSC{} level
    (Appendix~\ref{app:gap_mechanism}).
\end{enumerate}
\noindent The framework also explains why \BAEE{} avoids the gap: free
continuations impose no suffix ($s = \varnothing$), so the shift is zero
by construction.

\vspace{-0.5em}
\subsection{Generalization}
\label{sec:generalization}

The gap is robust across model families, scales, and benchmarks
(Table~\ref{tab:gpqa}, Figure~\ref{fig:combined}), but its
\emph{structure} varies in revealing ways.

\paragraph{Thinking mode amplifies the gap's precondition.}
Think models commit earlier (25--27\%) than NoThink (39--48\%),
generating 4$\times$ longer CoTs whose additional tokens are predominantly
post-commitment (all Think vs.\ NoThink contrasts $p{<}0.0001$,
permutation test).
Model size has no effect in Think mode ($p{=}0.76$) but is significant
in NoThink ($p{<}0.001$), suggesting that thinking tokens dominate
commitment dynamics regardless of capacity.

\begin{table}[t]
\centering
\caption{GPQA-Diamond vs MATH-500.
  The gap persists on a harder benchmark, but commitment occurs later
  and \BAEE{} reduction is correspondingly lower.}
\label{tab:gpqa}
\scriptsize
\setlength{\tabcolsep}{4pt}
\renewcommand{\arraystretch}{1.28}
\begin{tabular}{@{} l l c c c r c c c @{}}
\toprule
\textbf{Model} & \textbf{Benchmark} & \textbf{Acc.} & \textbf{Commit}
  & \textbf{Post-commit} & \textbf{CoT len}
  & \textbf{\textcolor{teal!80!black}{PSC@10\%}}
  & \textbf{\textcolor{red!65!black}{EFA@10\%}}
  & \textbf{Main-rollout red.} \\
\midrule
\multirow{2}{*}{8B-Think}
  & \cellcolor{gray!7}MATH-500     & \cellcolor{gray!7}76\% & \cellcolor{gray!7}27\% & \cellcolor{gray!7}73\% & \cellcolor{gray!7}2881 & \cellcolor{gray!7}\textcolor{teal!80!black}{70\%} & \cellcolor{gray!7}\textcolor{red!65!black}{25\%} & \cellcolor{gray!7}68\% \\
  & \cellcolor{teal!6}GPQA-Diamond & \cellcolor{teal!6}77\% & \cellcolor{teal!6}37\% & \cellcolor{teal!6}63\% & \cellcolor{teal!6}6884 & \cellcolor{teal!6}\textcolor{teal!80!black}{64\%} & \cellcolor{teal!6}\textcolor{red!65!black}{26\%} & \cellcolor{teal!6}49\% \\
\midrule
\multirow{2}{*}{8B-NoThink}
  & \cellcolor{gray!7}MATH-500     & \cellcolor{gray!7}90\% & \cellcolor{gray!7}48\% & \cellcolor{gray!7}52\% & \cellcolor{gray!7}696  & \cellcolor{gray!7}\textcolor{teal!80!black}{74\%} & \cellcolor{gray!7}\textcolor{red!65!black}{21\%} & \cellcolor{gray!7}73\% \\
  & \cellcolor{teal!6}GPQA-Diamond & \cellcolor{teal!6}74\% & \cellcolor{teal!6}38\% & \cellcolor{teal!6}62\% & \cellcolor{teal!6}1601 & \cellcolor{teal!6}\textcolor{teal!80!black}{60\%} & \cellcolor{teal!6}\textcolor{red!65!black}{34\%} & \cellcolor{teal!6}48\% \\
\midrule
\multirow{2}{*}{32B-Think}
  & \cellcolor{gray!7}MATH-500     & \cellcolor{gray!7}82\% & \cellcolor{gray!7}25\% & \cellcolor{gray!7}75\% & \cellcolor{gray!7}2878 & \cellcolor{gray!7}\textcolor{teal!80!black}{75\%} & \cellcolor{gray!7}\textcolor{red!65!black}{34\%} & \cellcolor{gray!7}72\% \\
  & \cellcolor{teal!6}GPQA-Diamond & \cellcolor{teal!6}81\% & \cellcolor{teal!6}33\% & \cellcolor{teal!6}67\% & \cellcolor{teal!6}6192 & \cellcolor{teal!6}\textcolor{teal!80!black}{73\%} & \cellcolor{teal!6}\textcolor{red!65!black}{34\%} & \cellcolor{teal!6}67\% \\
\midrule
\multirow{2}{*}{32B-NoThink}
  & \cellcolor{gray!7}MATH-500     & \cellcolor{gray!7}91\% & \cellcolor{gray!7}39\% & \cellcolor{gray!7}61\% & \cellcolor{gray!7}698  & \cellcolor{gray!7}\textcolor{teal!80!black}{76\%} & \cellcolor{gray!7}\textcolor{red!65!black}{29\%} & \cellcolor{gray!7}73\% \\
  & \cellcolor{teal!6}GPQA-Diamond & \cellcolor{teal!6}79\% & \cellcolor{teal!6}34\% & \cellcolor{teal!6}66\% & \cellcolor{teal!6}1448 & \cellcolor{teal!6}\textcolor{teal!80!black}{65\%} & \cellcolor{teal!6}\textcolor{red!65!black}{40\%} & \cellcolor{teal!6}52\% \\
\midrule
\multirow{2}{*}{GPT-OSS-120B}
  & \cellcolor{gray!7}MATH-500     & \cellcolor{gray!7}96\% & \cellcolor{gray!7}36\% & \cellcolor{gray!7}64\% & \cellcolor{gray!7}823  & \cellcolor{gray!7}\textcolor{teal!80!black}{92\%} & \cellcolor{gray!7}\textcolor{red!65!black}{22\%} & \cellcolor{gray!7}85\% \\
  & \cellcolor{teal!6}GPQA-Diamond & \cellcolor{teal!6}84\% & \cellcolor{teal!6}39\% & \cellcolor{teal!6}61\% & \cellcolor{teal!6}2351 & \cellcolor{teal!6}\textcolor{teal!80!black}{82\%} & \cellcolor{teal!6}\textcolor{red!65!black}{19\%} & \cellcolor{teal!6}61\% \\
\bottomrule
\end{tabular}
\end{table}

\paragraph{Task topology shapes the gap differently.}
On MATH-500, \PSC{} trajectories are \emph{non-monotone}: agreement
peaks around $f{\approx}0.50$ then \emph{declines}
(32B-Think: 81\%$\to$68\% from $f{=}0.50$ to $f{=}0.90$).
On GPQA-Diamond, \PSC{} increases \emph{monotonically}
(GPT-OSS: 71\%$\to$81\%).
This divergence reflects distinct commitment topologies:
MATH's short discrete answers allow early recoverability, but
subsequent CoT tokens can introduce perturbations that reduce
continuation agreement;
GPQA's sustained multi-step reasoning means each additional prefix
fraction contributes genuinely new information.
In practice, the non-monotone MATH pattern means early-only probing
suffices, while GPQA benefits from adaptive checkpoint sweeps
(Appendix~\ref{app:psc_mono}).

\vspace{-0.5em}
\begin{wrapfigure}{r}{0.44\linewidth}
\vspace{-0.5\baselineskip}
\centering
\includegraphics[width=\linewidth]{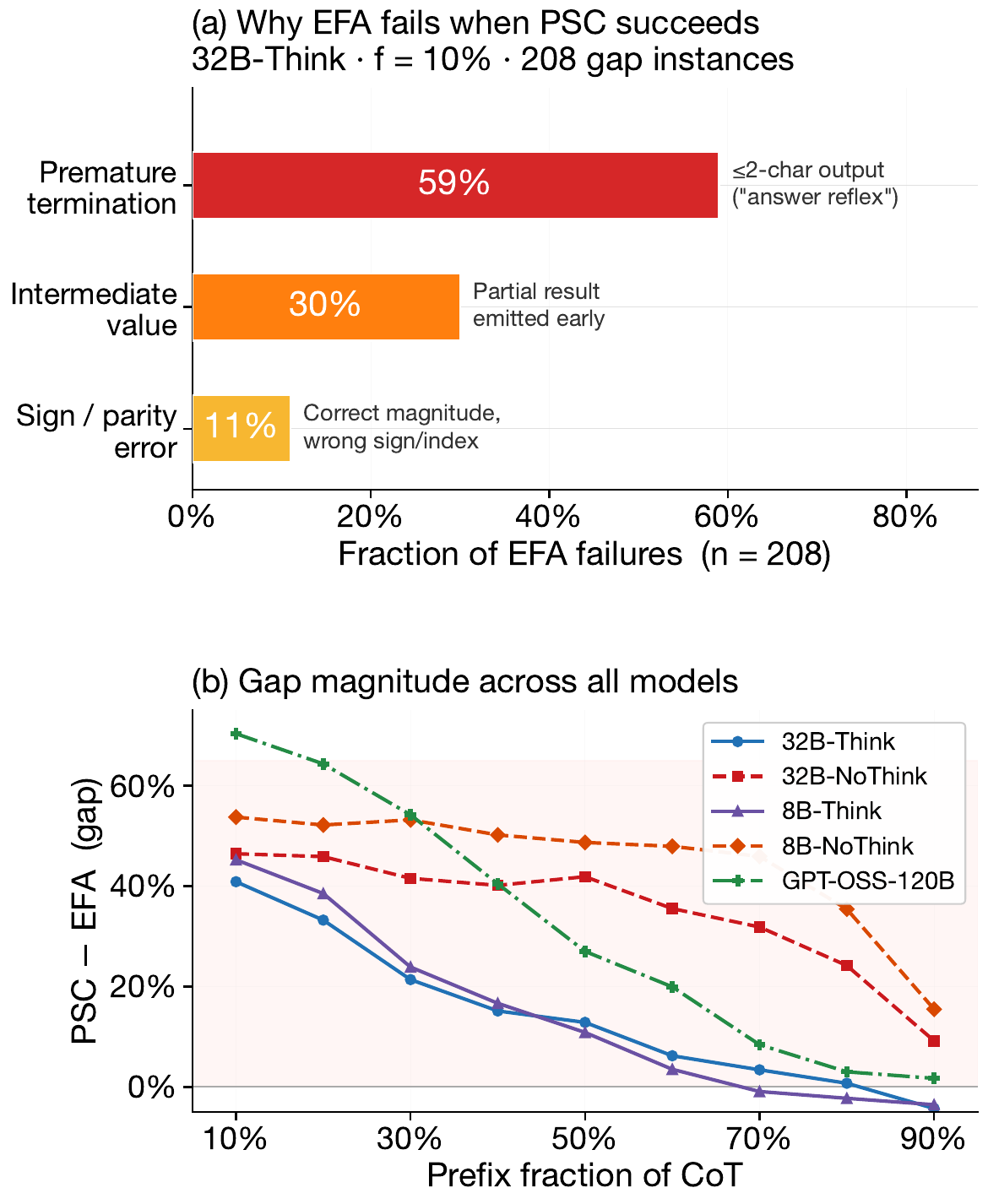}
\caption{(a)~Failure taxonomy for 208 gap instances.
  (b)~Gap magnitude across models.}
\label{fig:gap}
\vspace{-3\baselineskip}
\end{wrapfigure}

\paragraph{Cross-family validation.}
GPT-OSS-120B (a different model family and API) exhibits the
\emph{largest} gap ($\Delta_{0.1} = 70$\,pp) despite the highest
accuracy (96\%), confirming that the gap grows with model confidence
rather than reflecting weak commitment (Table~\ref{tab:gpqa}).

\vspace{-0.5em}
\paragraph{Code generation: the extreme case.}
On HumanEval~\citep{chen2021evaluating} (164 problems), post-commitment
fractions reach \textbf{85--88\%}, the highest across all benchmarks,
and the Think--NoThink gap collapses to $<$2\,pp (vs.\ 10--20\,pp on
math/science).
\BAEE{} produces its largest accuracy gains here: \textbf{+13.6\,pp} for
Think models (64.6\%$\to$78.2\%), providing direct evidence that extended
code generation corrupts initially correct algorithmic plans
(Appendix~\ref{app:humaneval}).

\begin{figure}[t]
\centering
\includegraphics[width=\linewidth]{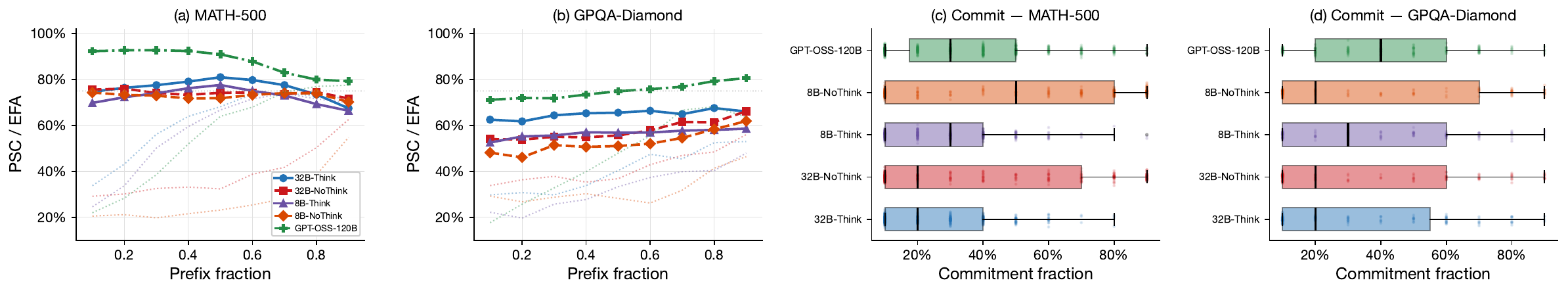}
\caption{MATH-500 vs GPQA-Diamond.
  (a,b)~PSC vs EFA across prefixes.
  (c,d)~Commitment distributions.}
\label{fig:combined}
\end{figure}

\subsection{Robustness and Supporting Evidence}
\label{sec:robustness}

We summarize three controls; full details are in the appendices.

\paragraph{Selection effect.}
\label{sec:selection}
On the common-solved subset,
the Think--NoThink commitment gap persists:
8B ($n{=}372$): 26\% vs 46\% ($p{<}0.0001$);
32B ($n{=}402$): 24\% vs 38\% ($p{<}0.0001$),
ruling out problem-difficulty selection as a confound
(Appendix~\ref{app:stats}).

\vspace{-0.5em}
\paragraph{False positives.}
\label{sec:fp_main}
Among 2{,}912 high-\PSC{} instances across MATH-500 and GPQA-Diamond,
63 (2.2\%) are false positives (wrong answer, high agreement).
FP trajectories are distinguishable from TPs: they start low
(PSC@10\% = 0.33 vs 0.85), oscillate heavily, and peak late,
enabling trajectory-based filtering that eliminates 74\% of FPs while
retaining 89\% of TPs (Appendix~\ref{app:fp_trajectory}).
\vspace{-0.5em}

\paragraph{Prefix perturbation.}
\label{sec:perturbation}
Replacing 30\% of prefix tokens with random vocabulary items reduces
\PSC{} by only 2.8\,pp at $f{=}0.10$, confirming that the answer state
is deeply embedded (Appendix~\ref{app:perturbation}).
\vspace{-0.5em}

\paragraph{Entropy corroboration.}
\label{sec:entropy}
Post-commitment entropy \emph{rises} for Think models
($1.38$--$1.44\times$ pre-commit levels) but \emph{falls} for NoThink
($0.70$--$0.87\times$; Figure~\ref{fig:entropy}c), confirming a clean
asymmetry between performative and convergent generation regimes
(Appendix~\ref{app:entropy}).

\begin{figure}[t]
\centering
\includegraphics[width=\linewidth]{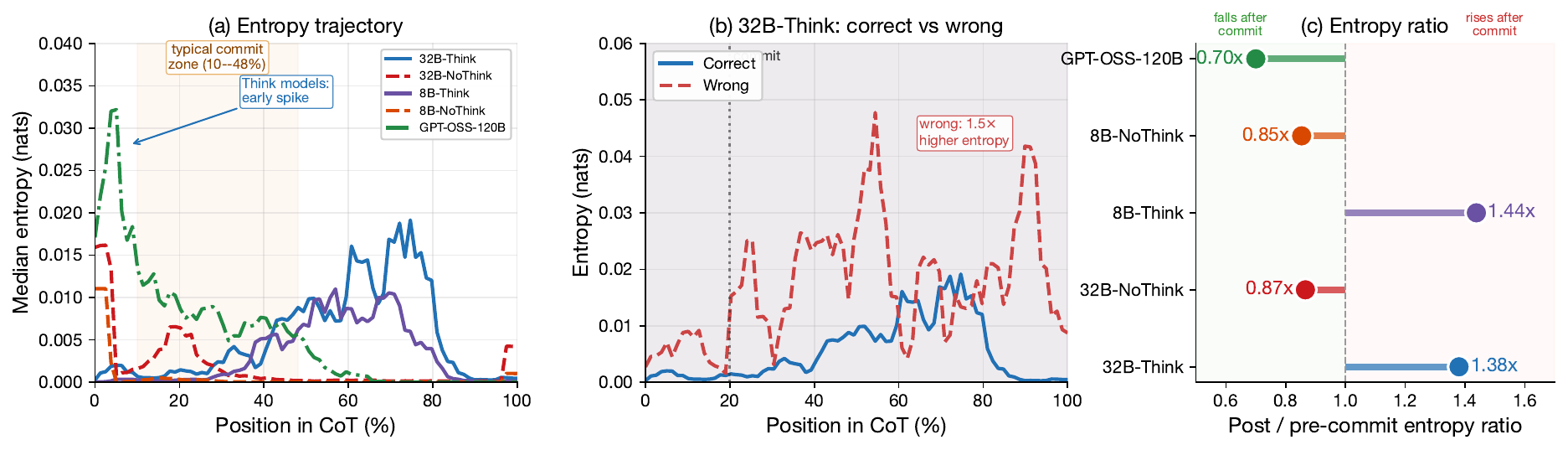}
\caption{Entropy analysis.
  (a)~Per-token entropy along the CoT.
  (b)~Wrong problems (dashed) show higher entropy.
  (c)~Post/pre-commit ratio separates Think ($>$1) from NoThink ($<$1).}
\label{fig:entropy}
\end{figure}

\subsection{BAEE: From Detection to Early Exit}
\label{sec:baee}

\subsubsection{Calibrated Threshold Selection}
\label{sec:calibration}

\begin{table}[t]
\centering
\caption{Accuracy / main-rollout reduction for each \BAEE{} strategy.
         ``Reduction'' = serial main-rollout tokens avoided (\emph{not} total billed tokens;
         see \S\ref{sec:cost_benefit}).
         \PSC{}-8 all probes every checkpoint; \PSC{}-8 adaptive stops at first trigger.
         Format: \textit{Acc.\,/\,Reduction}.}
\label{tab:baee}
\footnotesize
\setlength{\tabcolsep}{4.5pt}
\renewcommand{\arraystretch}{1.30}
\begin{tabular}{@{} l c c cccc @{}}
\toprule
\textbf{Strategy} & \textbf{GT?} & \textbf{Calls}
  & \textbf{32B-Think} & \textbf{32B-NoThink} & \textbf{8B-Think} & \textbf{8B-NoThink} \\
\midrule
\rowcolor{gray!7}
Full CoT        & ---          & 1          & 82\% / ---  & 91\% / ---  & 76\% / ---  & 90\% / ---  \\
\rowcolor{red!7}
Na\"ive \EFA{}  & No           & $\sim$6    & \textcolor{red!70!black}{34\%} / 90\% & \textcolor{red!70!black}{38\%} / 90\% & \textcolor{red!70!black}{35\%} / 90\% & \textcolor{red!70!black}{28\%} / 90\% \\
\rowcolor{gray!5}
EFA-oracle      & \textit{Yes} & $\sim$6    & 82\% / 67\% & 91\% / 48\% & 76\% / 67\% & 90\% / 38\% \\
\rowcolor{teal!12}
\textbf{\PSC{}-8 all} & \textbf{No} & \textbf{72}
  & \textbf{86\% / 74\%} & \textbf{92\% / 78\%} & \textbf{81\% / 70\%} & \textbf{91\% / 77\%} \\
\rowcolor{teal!5}
\PSC{}-8 adaptive & No & 9 (med.)
  & 82\% / 72\% & 91\% / 73\% & 76\% / 68\% & 90\% / 73\% \\
\bottomrule
\end{tabular}
\end{table}

\begin{wraptable}{r}{0.545\linewidth}
\vspace{-0.5\baselineskip}
\centering
\caption{Calibrated threshold $\theta^*$ and held-out test performance on MATH-500
  (250-problem split) and GPQA-Diamond (99-problem split).
  $\theta^* = 1.0$ for 4/5 MATH models; GPQA selects lower thresholds, consistent with
  its harder problems relaxing the FP constraint.}
\label{tab:calib}
\footnotesize\setlength{\tabcolsep}{4pt}\renewcommand{\arraystretch}{1.25}
\begin{tabular}{@{} l
    >{\columncolor{blue!5}}c >{\columncolor{blue!5}}c >{\columncolor{blue!5}}c
    >{\columncolor{teal!5}}c >{\columncolor{teal!5}}c >{\columncolor{teal!5}}c @{}}
\toprule
& \multicolumn{3}{c}{\cellcolor{blue!10}\textbf{MATH-500}}
& \multicolumn{3}{c}{\cellcolor{teal!10}\textbf{GPQA-Diamond}} \\
\cmidrule(lr){2-4}\cmidrule(l){5-7}
\textbf{Model}
  & $\boldsymbol{\theta^*}$ & $\boldsymbol{\Delta}$\textbf{Acc} & \textbf{Red.}
  & $\boldsymbol{\theta^*}$ & $\boldsymbol{\Delta}$\textbf{Acc} & \textbf{Red.} \\
\midrule
\rowcolor{blue!4!teal!2} 32B-Think    & 1.000 & $+$1.8\,pp & 67\% & 0.875 & $+$1.0\,pp & 54\% \\
                         32B-NoThink  & 0.875 & $+$0.8\,pp & 71\% & 1.000 & $\pm$0      & 36\% \\
\rowcolor{blue!4!teal!2} 8B-Think     & 1.000 & $+$1.2\,pp & 63\% & 0.750 & $+$3.0\,pp & 47\% \\
                         8B-NoThink   & 1.000 & $+$1.4\,pp & 67\% & 0.750 & $+$3.0\,pp & 47\% \\
\rowcolor{blue!4!teal!2} GPT-OSS-120B & 1.000 & $+$0.4\,pp & 81\% & 0.625 & $+$4.0\,pp & 72\% \\
\bottomrule
\end{tabular}
\vspace{-1\baselineskip}
\end{wraptable}

\textbf{\PSC{}-8 all} is the accuracy-optimal strategy (Table~\ref{tab:baee}):
it \emph{improves} accuracy over full-CoT by \textbf{4--5\,pp on Think models}
(32B-Think: $82\%\to86\%$; 8B-Think: $76\%\to81\%$), while still achieving
\textbf{70--78\% main-rollout reduction}.
NoThink models gain 1\,pp.
The improvement is not from majority voting alone---it reflects
\emph{overthinking prevention}: early exit stops the model before
post-commitment tokens overwrite initially correct answers (\S\ref{sec:aggressive}).
For cost-sensitive settings, \textbf{\PSC{}-8 adaptive} provides the same
68--73\% reduction at a median of only 9 API calls (67--92\% trigger at
$f{=}0.10$), matching full-CoT accuracy exactly.
Under calibrated thresholds on held-out test sets (Table~\ref{tab:calib}),
all models maintain $\Delta\text{Acc} \geq 0$.

To verify these are not artifacts of post-hoc tuning, we also report a
\emph{calibrated} protocol.
We partition the 500 MATH problems into a calibration set (first 250) and a
held-out test set (last 250), sweep
$\theta \in \{0.500, 0.625, 0.750, 0.875, 1.000\}$, and select the smallest
$\theta$ satisfying: (i)~BAEE accuracy $\geq$ full-CoT baseline,
and (ii)~wrong-problem proxy FP rate $\leq 5\%$.

The calibrated procedure selects $\theta^* = 1.0$ for 4/5 MATH models, stricter than
the main-text operating points, yet all models match or exceed full-CoT accuracy, with
\textbf{63--81\% main-rollout reduction} on MATH and \textbf{36--72\%} on GPQA.
$\theta^*$ transfers across benchmarks only partially: GPQA selects lower values
(0.625--0.875 vs 0.875--1.0), confirming that \textbf{per-benchmark calibration is
needed} for deployment, though the core reduction--accuracy trade-off is robust in
both regimes.

\paragraph{Why na\"ive \EFA{} fails.}
Na\"ive \EFA{} (exiting at the first non-empty \EFA{} answer) is catastrophic:
it drops accuracy by 41--62 percentage points (Table~\ref{tab:baee}), a direct
empirical consequence of the detection--extraction gap (\S\ref{sec:gap}).

\subsubsection{Prefix-Free Baselines}
\label{sec:baselines}

To disentangle prefix state from majority voting, we evaluate three prefix-free
controls:

\begin{itemize}
  \item \textbf{SC-8-full}: 8 complete CoTs from scratch, majority vote.
        Matches \PSC{} call count but uses full per-call budget.
  \item \textbf{SC-8-budget}: 8 continuations from scratch, budget-matched to \PSC{}.
  \item \textbf{Single-budget}: 1 generation from scratch at the \PSC{} total budget.
\end{itemize}

On MATH-500, SC-8-full matches \PSC{}-8 adaptive accuracy within 0--2\,pp.
On GPQA, SC-8-full drops 14--20\,pp below \BAEE{}
(Appendix~\ref{app:token_matched}).
A token-matched comparison confirms \BAEE{} achieves +20--37\,pp higher accuracy
than SC-8 at 1.6--2.6$\times$ fewer total tokens;
see Discussion (\S\ref{sec:discussion}) for interpretation.

\subsubsection{The Cost: Serial-to-Parallel Compute Trade-off}
\label{sec:cost_benefit}

\paragraph{Serial-to-parallel conversion.}
\BAEE{} converts \emph{depth-heavy} sequential reasoning into \emph{width-heavy}
parallel verification: serial CoT tokens are replaced by parallel continuation
probes that execute concurrently.
\vspace{-0.5em}

\paragraph{Observation} (Serial--parallel trade-off under black-box access).
\emph{In the absence of internal-state access, any early-exit policy that maintains
accuracy must replace serial reasoning tokens with external sampling-based
verification.
Consequently, black-box early exit incurs a systematic shift from depth (sequential
tokens) to width (parallel samples), reducing the longest sequential decoding path
while typically increasing total token usage.}

\paragraph{Full token accounting and when to use BAEE.}
\begin{wraptable}{r}{0.44\linewidth}
\vspace{-0.4\baselineskip}
\centering
\caption{Compute redistribution under \BAEE{} (MATH-500). Estimated totals
  use measured continuation lengths (Appendix~\ref{app:psc_raw}).
  GPQA-Diamond ratios are lower (3.1--5.0$\times$); see Appendix~\ref{app:token_matched}.}
\label{tab:total_cost}
\footnotesize\setlength{\tabcolsep}{4pt}\renewcommand{\arraystretch}{1.22}
\begin{tabular}{@{} l
    >{\columncolor{blue!5}}r
    >{\columncolor{teal!5}}r
    >{\columncolor{blue!5}}r
    >{\columncolor{teal!5}}r @{}}
\toprule
\multicolumn{5}{@{}c@{}}{\cellcolor{blue!10}\textbf{MATH-500}} \\
\midrule
\textbf{Model}
  & \textbf{CoT} & \textbf{Ser.\ red.}
  & \textbf{Est.\ tot.}
  & \textbf{Ratio} \\
\midrule
\rowcolor{blue!4!teal!2}
32B-Think    & 2\,879 & 63\% & 11\,228 & 3.9$\times$ \\
32B-NoThink  &   699 & 55\% &  3\,006 & 4.3$\times$ \\
\rowcolor{blue!4!teal!2}
8B-Think     & 2\,881 & 57\% & 10\,372 & 3.6$\times$ \\
8B-NoThink   &   697 & 55\% &  2\,927 & 4.2$\times$ \\
\rowcolor{blue!4!teal!2}
GPT-OSS-120B &   823 & 76\% &  4\,115 & 5.0$\times$ \\
\bottomrule
\end{tabular}
\vspace{-1\baselineskip}
\end{wraptable}

Committed-prefix continuations converge quickly (1.05$\times$ remaining CoT on
MATH-500; Appendix~\ref{app:psc_raw}), yielding total-token ratios of
\textbf{3.6--5.0$\times$} on MATH-500 and \textbf{3.1--5.0$\times$} on
GPQA-Diamond (Appendix~\ref{app:token_matched}), where harder problems trigger
later and leave proportionally less parallel overhead.
For context, SC-8-full costs a fixed
\textbf{8.0$\times$} on all benchmarks; \BAEE{} is therefore
\textbf{37--61\% more token-efficient} than SC-8-full while achieving
substantially higher accuracy (Appendix~\ref{app:token_matched}).

\BAEE{} reduces latency under parallel execution but increases total tokens by
\textbf{3--5$\times$} (3.1--5.0$\times$ across benchmarks), making it unsuitable for token-budget-constrained settings.
With adaptive stopping, the median cost is $N+1 = 9$ API calls (67--92\% trigger at
$f=0.10$).
The worst case (all 9 checkpoints) costs $9 \times 8 + 1 = 73$ calls, but this is rare:
most non-triggering problems fail the threshold at every checkpoint and fall back to
the full CoT answer at zero additional cost.

\subsubsection{Aggressive Operating Point}
\label{sec:aggressive}

As a secondary analysis, we evaluate an aggressive offline majority-vote operating
point at $\theta=0.625$ (5/8 trigger).
On thinking models, this regime \emph{improves} accuracy while still truncating
substantial serial generation: 8B-Think gains \textbf{5.8 percentage points}
(75.6\%~$\to$~81.4\%) while reducing 70\% of serial main-rollout generation.
Direct evidence for the overthinking mechanism: among the 29 problems where \BAEE{}
corrects 8B-Think errors (wrong under full CoT, correct under early exit),
we verified that the full-CoT rollout initially produces the correct answer
before overwriting it in later tokens, confirming that post-commitment generation
actively harms accuracy on these problems.
Because this aggressive regime mixes a different risk profile than the main
deployment setting, full details are in Appendix~\ref{app:baee_overthinking}.

\section{Discussion}
\label{sec:discussion}

\paragraph{Implications of the gap.}
The mechanistic account in \S\ref{sec:gap} (reachability vs.\ constrained
decoding) implies that the gap is not a prompt artifact but a structural property of
partially formed reasoning states, and it imposes a design constraint: any early-exit method
based on forced extraction inherits this failure mode, whereas \BAEE{} avoids it by never
imposing the distribution shift.

\paragraph{Probe reliability and cost.}
\label{sec:tradeoff}
Behavioral probing involves a three-way trade-off among \emph{granularity}
(checkpoint density $|F|$), \emph{reliability} ($N$ samples per checkpoint),
and \emph{cost} (API calls).
Our main grid ($|F|{=}9$, $N{=}8$) costs 72 calls in the worst case, but only
9 at the median via adaptive stopping (67--92\% trigger at $f{=}0.10$).
Finer grids (Appendix~\ref{app:fine_grained}) show \PSC{} $>$90\% at the 2\%
prefix, confirming the 10\% grid overestimates rather than inflates
commitment. \PSC{} estimates Bernoulli probability with
CI~$\pm\!0.30$ at $N{=}8$, bounding spurious triggers at $<$1\% empirically.

\paragraph{Relationship to white-box probing.}
\label{sec:whitebox_comparison}
Our \PSC{}-based commitment fractions (25--48\%) are \emph{upper bounds} on
latent commitment (the hidden state must encode the answer before continuations
can recover it), consistent with white-box probes~\citep{boppana2026reasoning}
reporting 20--30\% on MMLU (a 5--15\,pp ``behavioral lag'').
Because white-box probes bypass generation entirely, they cannot observe the
detection--extraction gap; our approach reveals it as a first-class phenomenon
and produces a deployable policy for any API-accessible model.

\paragraph{What does the prefix contribute?}
On MATH-500, cold-start SC-8 nearly matches \PSC{}@10\% (+1.4\,pp;
Appendix~\ref{app:null_prefix}): the prefix contributes \emph{efficiency}
(37--61\% fewer total tokens than SC-8-full) rather than accuracy.
On GPQA-Diamond, the prefix is essential: SC-8-full drops 14--20\,pp below
\BAEE{} (Appendix~\ref{app:token_matched}), confirming that \BAEE{}'s
unique value emerges on difficult tasks where prefix computation is
irreplaceable by cold-start sampling.

\paragraph{Falsifiability of ``structural.''}
By \emph{structural} we mean the gap arises from suffix-induced
distributional shift on the model's \emph{learned policy}, not a dataset
artifact.
Suffix ablation shows the gap tracks the shift induced by the forcing
suffix (Spearman $\rho{=}1.0$ rank-calibration; \S\ref{sec:gap}), not
idiosyncratic prompt formatting---a relationship expected under
suffix-induced policy shift and less natural under purely cosmetic
evaluation quirks.

\paragraph{Broader implications.}
The gap suggests that current CoT training incentivizes models to encode
answers in latent states earlier than their generation policies can
externalize them, a form of \emph{capability--elicitation misalignment}
at the sequence level.
If confirmed in larger models and more diverse domains, this has
consequences beyond efficiency: it implies that evaluation protocols that
rely on forced extraction (common in benchmarking and safety auditing)
may systematically underestimate what models have already computed.
Conversely, the TV view (Appendix~\ref{app:gap_mechanism}) suggests
designing extraction to minimize suffix-induced shift, potentially closing
the gap without sacrificing structured reasoning~\citep{wang2025text2grad}.

\vspace{-0.5em}
\section{Conclusion}
\label{sec:conclusion}
\vspace{-0.5em}
Across five model configurations and three benchmarks, the majority of
CoT tokens are generated after the answer is already recoverable, yet
forcing the model to state its answer immediately fails on nearly half of
these recoverable problems.
This \emph{detection--extraction gap} is a structural property of
partially-formed reasoning states, bounded by the distributional shift
that forced extraction imposes.
\BAEE{} exploits this structure by using free continuations for both
detection and extraction, truncating 70--78\% of serial generation while
improving accuracy by 1--5\,pp; for thinking-mode models, gains reach
5.8\,pp on math/science and 13.6\,pp on code by preventing
post-commitment overthinking.
The consistency of these findings across math, science, and code
generation suggests that the gap is a widespread property of current CoT
models, and that understanding it, not just optimizing around it, is
essential for efficient reasoning at inference time.

\bibliographystyle{plainnat}
\bibliography{references}

\appendix

\section{LLM Usage for Manuscript Preparation}
\label{app:llm_usage}

The authors used a large language model assistant solely for \emph{language polishing} and minor typographical formatting.
It was not used to propose methodology, experiments, statistical analyses, or numerical results; all such content was authored and verified by the human authors.

\section{Limitations}
\label{sec:limitations}

\begin{enumerate}
  \item \textbf{Recoverability vs.\ commitment}: \PSC{} measures behavioral
        recoverability, which upper-bounds latent commitment (\S\ref{sec:whitebox_comparison}).
        Multiple controls (difficulty stratification, common-solved subsets,
        and three-benchmark validation) partially address this gap.
  \item \textbf{Temporal resolution}: The main experiments use 9 checkpoints (10\%--90\%).
        To assess sensitivity, we run finer-grained probing on 50 problems with checkpoints
        at \{2\%, 4\%, 5\%, 6\%, 8\%, 10\%, 12\%, 15\%, 20\%, 25\%, 30\%, 40\%, 50\%\}
        (Appendix~\ref{app:fine_grained}).
        PSC agreement at 2\% already reaches 90\%, confirming that the 10\% grid
        does not artificially inflate post-commitment fractions; if anything, commitment occurs
        even earlier than the main grid captures.
  \item \textbf{Benchmark scope}: Results span MATH-500, GPQA-Diamond,
        and HumanEval (Appendices~\ref{app:difficulty_aime},~\ref{app:humaneval}),
        covering math, science, and code generation.
        Competition-level mathematics and common-sense reasoning remain future work.
  \item \textbf{White-box comparison}: Our indirect comparison
        (\S\ref{sec:whitebox_comparison}) shows estimates consistent with
        white-box reports; a direct comparison on the same model is future work.
  \item \textbf{EFA suffix bias}: 9--16\% of \EFA{} probes on unsolvable problems return
        ``correct'' answers; this is accounted for in our gap analysis and does not
        affect \PSC{}-based metrics.
  \item \textbf{Total-token cost}: \BAEE{} trades serial depth for parallel width,
        increasing total tokens 3--5$\times$ under empirical continuation lengths
        (MATH-500: 3.6--5.0$\times$; GPQA-Diamond: 3.1--5.0$\times$;
        \S\ref{sec:cost_benefit}), always below SC-8-full's fixed 8.0$\times$.
        Under parallel execution (standard in API deployments), the latency reduction
        (63--76\%) is the operationally relevant metric; in token-budgeted settings,
        BAEE is not the appropriate tool.
\end{enumerate}

\section{Statistical Methods}
\label{app:stats}

All hypothesis tests are conducted at the \textbf{problem level}: each problem
contributes one scalar commitment fraction per model configuration.
Rollouts and checkpoints are used only to \emph{construct} that scalar and do not
appear as independent observations in inferential tests.

\paragraph{Inference procedures.}
We report 95\% bootstrap CIs (10,000 resamples) for commitment fractions, and
two-sided permutation tests (100,000 permutations) for comparing commitment
distributions between model configurations.
Mann-Whitney U is reported as a non-parametric alternative; the permutation test
is our primary significance measure as it makes no distributional assumptions.

\paragraph{Paired analysis for common-solved subsets.}
For common-solved analyses (same problem solved by both variants), data are paired by
problem ID.
In addition to unpaired permutation tests, we run \textbf{paired sign-flip permutation}
tests on within-problem commitment differences (Think minus NoThink), and paired
bootstrap CIs for the mean paired difference.
Results remain highly significant on the full 500-problem runs:
8B paired permutation $p<10^{-5}$ ($n=248$ common-solved, mean paired difference
$-0.228$); 32B paired permutation $p<10^{-5}$ ($n=305$ common-solved, mean paired
difference $-0.170$).

\paragraph{Multiple comparisons.}
For the primary family of commitment-gap tests reported in the main text
(H2 at 8B/32B, common-solved at 8B/32B, and size effects in Think/NoThink),
Holm-Bonferroni correction preserves all non-null findings.
Adjusted $p$-values are:
0.006 (H2-32B), 0.0006 (H2-8B), 0.006 (common-solved-32B),
0.0006 (common-solved-8B), 0.0016 (size effect in NoThink), and 0.76
(size effect in Think; non-significant).

\paragraph{Re-grading procedure.}
The original \EFA{} pipeline stripped trailing characters sequentially
(\texttt{rstrip('\}').rstrip('.')}), which fails on \texttt{"answer\}."} where
the \texttt{.} blocks \texttt{\}} removal.
The fix uses simultaneous stripping (\texttt{rstrip('\}.')}), correctly handling all edge cases.
This affected approximately 12--34 evaluations per model across the 500-problem runs
(32B-Think: $\sim$170 extrapolated from pilot; 32B-NoThink: 12; 8B-Think: 9; 8B-NoThink: 11;
GPT-OSS-120B: 24).

\section{Supplementary: Aggressive Majority-Vote BAEE}
\label{app:baee_overthinking}

This section reports a \emph{secondary} offline simulation under a more aggressive
majority-vote operating point ($\theta=0.625$, i.e., 5/8 trigger).
It is not the primary deployment recommendation; the main text focuses on conservative
thresholds chosen for robustness.

\begin{table}[ht]
\centering
\caption{\PSC{}-8 BAEE accuracy and main-rollout reduction under \emph{offline}
  majority-vote simulation at $\theta = 0.625$ (5/8 trigger).
  \textit{Base} is full-CoT accuracy; $\Delta$\textit{Acc} is the change from early exit.
  ``Token reduction'' refers to serial main-rollout tokens only.
  All models: 500 problems.}
\label{tab:baee_majority}
\footnotesize
\setlength{\tabcolsep}{4.6pt}
\renewcommand{\arraystretch}{1.12}
\begin{tabular}{@{} l ccc cc @{}}
\toprule
\textbf{Model} & \textbf{Base acc.} & \textbf{BAEE acc.} & $\boldsymbol{\Delta}$\textbf{Acc}
  & \textbf{Main-rollout red.} & \textbf{Exit Rate} \\
\midrule
\rowcolor{blue!6}  8B-Think     & 75.6\% & \textbf{81.4\%} & \textbf{$+$5.8\,pp} & 69.8\% & 80.6\% \\
                   8B-NoThink   & 89.8\% & 90.6\%           & $+$0.8\,pp           & 77.0\% & 89.4\% \\
\rowcolor{blue!6}  32B-Think    & 82.0\% & \textbf{86.4\%} & \textbf{$+$4.4\,pp} & 74.3\% & 85.6\% \\
                   32B-NoThink  & 91.4\% & 92.0\%           & $+$0.6\,pp           & 78.3\% & 90.6\% \\
\rowcolor{green!4} GPT-OSS-120B & 95.6\% & \textbf{96.2\%} & $+$0.6\,pp           & \textbf{85.6\%} & \textbf{96.0\%} \\
\bottomrule
\end{tabular}
\end{table}

Table~\ref{tab:baee_majority} shows that this aggressive operating point does not merely
preserve accuracy: for thinking-mode models, it can \emph{increase} it.
8B-Think gains \textbf{5.8 percentage points} (75.6\%~$\to$~81.4\%) while
reducing 70\% of serial main-rollout generation.

\paragraph{Mechanism: PSC interrupts overthinking.}
As discussed in \S\ref{sec:baee}, thinking models reach recoverability at $\approx 25\%$ of the CoT
but continue generating, during which post-commitment tokens can overwrite the initially
correct answer~\citep{chen2024not}.
Early exit prevents this degradation.
The effect scales with post-commitment fraction: 8B-Think gains the most (+5.8\,pp,
73\% post-commitment), while NoThink variants ($\leq$1\,pp) and GPT-OSS-120B (+0.6\,pp, 64\%
post-commitment) gain proportionally less.

\begin{figure}[t]
\centering
\includegraphics[width=\linewidth]{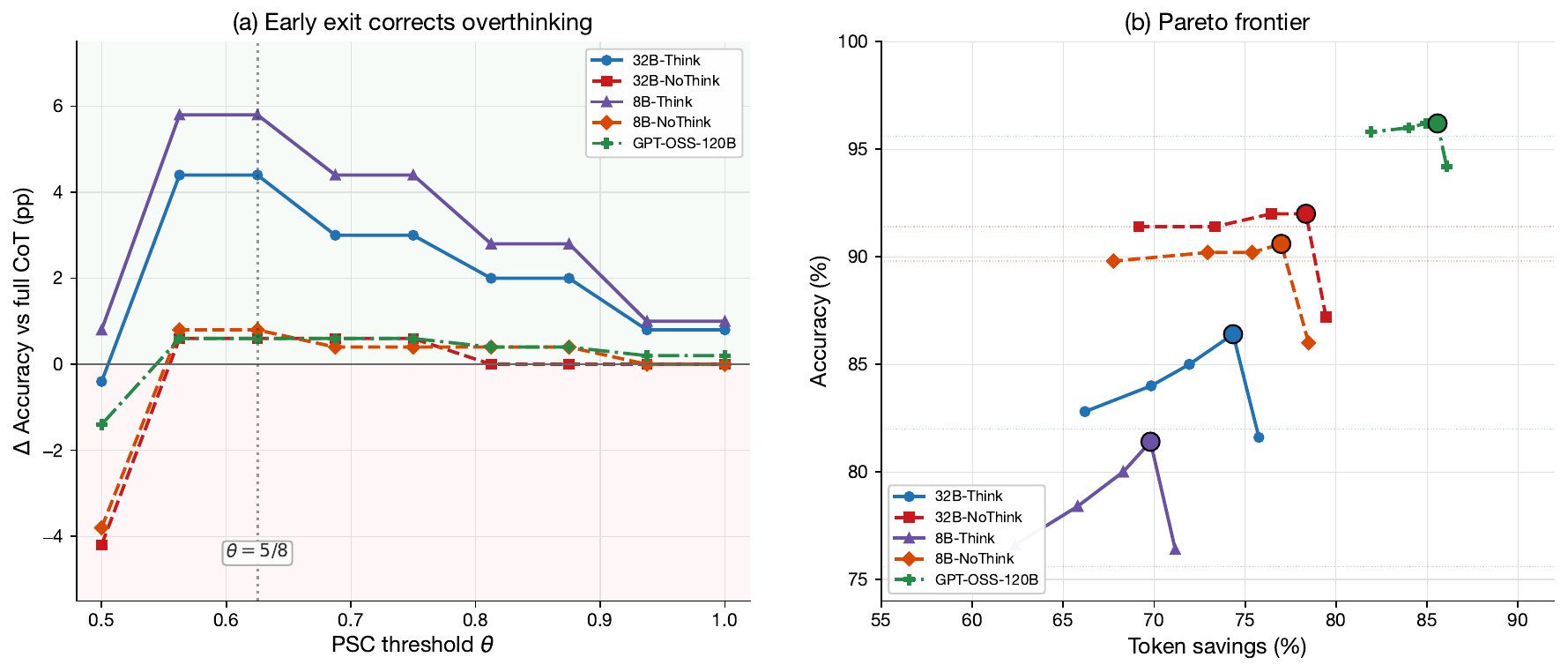}
\caption{Offline majority-vote BAEE under aggressive thresholds.
  (a)~Accuracy change vs full CoT across $\theta$ thresholds: Think models (solid lines)
  gain up to +5.8~pp; NoThink models (dashed) can degrade for low $\theta$.
  (b)~Pareto frontier of accuracy vs main-rollout reduction; large dots mark $\theta=0.625$.
  Dotted lines show full-CoT baseline accuracy for each model.}
\label{fig:overthinking}
\end{figure}

\begin{figure}[t]
\centering
\includegraphics[width=0.85\linewidth]{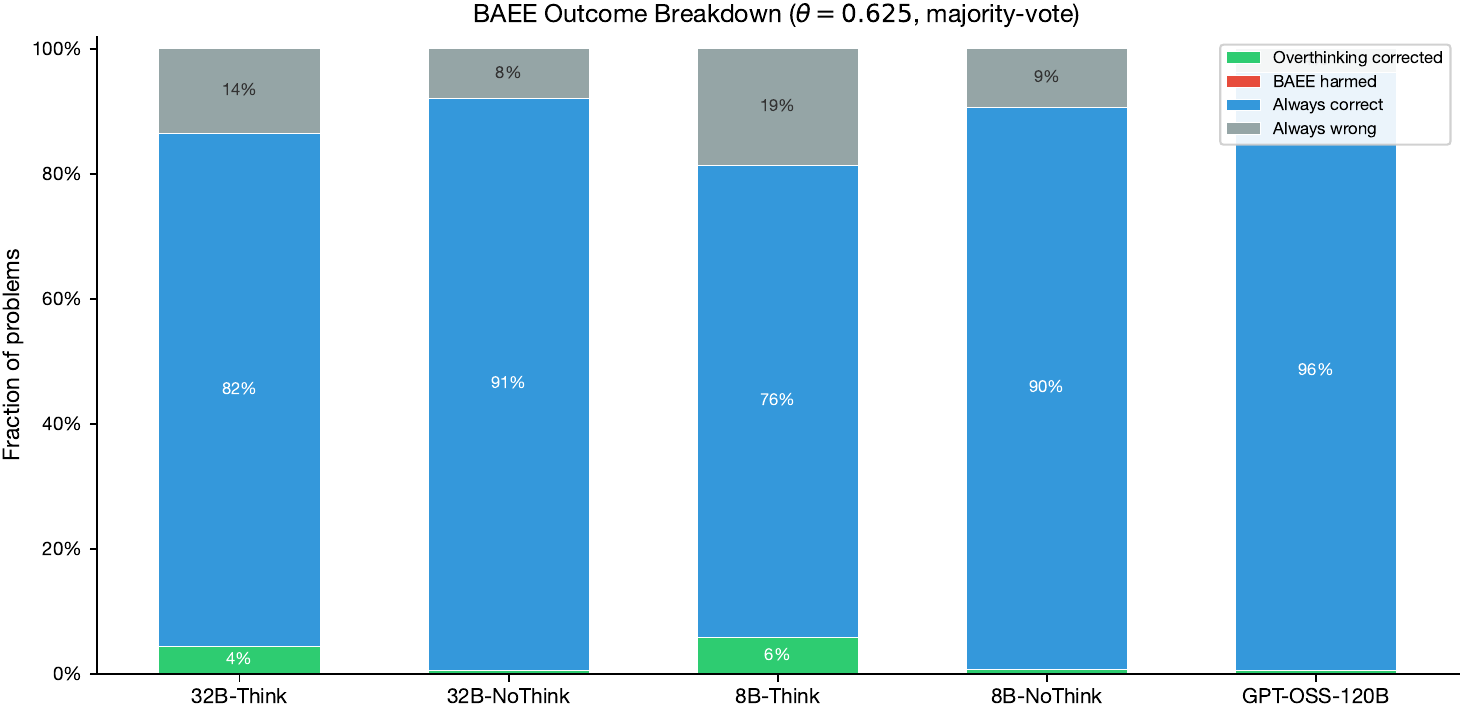}
\caption{Outcome breakdown at $\theta=0.625$ (offline majority-vote simulation).
  Green = overthinking corrected (wrong under full CoT, correct under BAEE);
  red = BAEE harmed (correct under full CoT, wrong under BAEE);
  blue = always correct; gray = always wrong.}
\label{fig:baee_outcomes}
\end{figure}

\section{PSC False Positive Analysis}
\label{app:fp}

\subsection{PSC on Unsolvable Problems}

\begin{table}[ht]
\centering
\caption{Maximum \PSC{} accuracy across all checkpoints for problems with 0/4 correct
         rollouts. NoThink models show nonzero proxy FP at $\theta=0.75$.
         ``N wrong'' uses 0/4 correct rollouts as a proxy for unsolvable instances.}
\label{tab:fp}
\footnotesize
\setlength{\tabcolsep}{4.5pt}
\renewcommand{\arraystretch}{1.10}
\begin{tabular}{@{} l c c c c @{}}
\toprule
\textbf{Model} & \textbf{N wrong} & \textbf{Max \PSC{}} & \textbf{\PSC{} $\geq 50\%$} & \textbf{\PSC{} $\geq 75\%$} \\
\midrule
\rowcolor{teal!6}  32B-Think    &  17 & 0.50 & 1 (at 25\%) & \textbf{\textcolor{teal!70!black}{0}} \\
\rowcolor{red!4}   32B-NoThink  &  43 & 0.75 & 3           & \textbf{\textcolor{red!65!black}{3}} \\
\rowcolor{teal!6}  8B-Think     &  20 & 0.31 & 0           & \textbf{\textcolor{teal!70!black}{0}} \\
\rowcolor{red!4}   8B-NoThink   &  51 & 0.88 & 6           & \textbf{\textcolor{red!65!black}{2}} \\
\rowcolor{teal!6}  GPT-OSS-120B &   3 & 0.19 & 0           & \textbf{\textcolor{teal!70!black}{0}} \\
\bottomrule
\end{tabular}
\end{table}

\PSC{} accuracy does not reach 75\% on unsolvable problems for Think models or GPT-OSS-120B.
However, among the larger 500-problem NoThink sets, 3/43 wrong problems (7\%) for 32B-NoThink
and 2/51 (4\%) for 8B-NoThink reach $\PSC{} \geq 0.75$.
Raising the threshold to $\theta = 0.875$ eliminates these cases for 32B-NoThink and
retains 2/51 for 8B-NoThink; at $\theta=1.0$, the 8B-NoThink proxy FP rate falls to 0/51.
\BAEE{} simulation accuracy is preserved at 91.4\% for 32B-NoThink and 89.8\% for
8B-NoThink, matching the full-CoT baseline.

Note that in deployment, \PSC{} measures self-agreement rather than accuracy.
Since correct answers by definition agree with each other, \PSC{} accuracy $\leq$
self-agreement, making our threshold conservative.
The remaining risk (that wrong problems have high self-agreement on an incorrect
answer) is addressed empirically in Appendix~\ref{app:psc_raw}.

\subsection{Threshold Sweep and Operating Points}
\label{app:theta_sweep}

To reduce post-hoc thresholding concerns, we report a fixed sweep over
$\theta \in \{1/8, 2/8, \ldots, 1\}$ (aligned with 8-sample \PSC{} granularity).
For each $\theta$, we compute BAEE simulation accuracy, mean savings, and proxy FP rate
on the 0/4-wrong subset.

\begin{table}[ht]
\centering
\caption{NoThink threshold sweep (500-problem runs): BAEE simulation accuracy /
         mean main-rollout reduction / proxy FP rate on 0/4-wrong subset.}
\label{tab:theta_sweep_nothink}
\footnotesize
\setlength{\tabcolsep}{4.6pt}
\renewcommand{\arraystretch}{1.10}
\begin{tabular}{@{} l c c c @{}}
\toprule
\textbf{Setting ($\theta$)} & \textbf{Accuracy} & \textbf{Mean savings} & \textbf{Proxy FP rate} \\
\midrule
                   32B-NoThink, $\theta=0.750$  & 91.4\% & \textbf{76.4\%} & \textcolor{red!65!black}{7.0\%} (3/43) \\
\rowcolor{teal!5}  32B-NoThink, $\theta=0.875$ & 91.4\% & 73.3\% & \textbf{\textcolor{teal!70!black}{0.0\%}} (0/43) \\
\rowcolor{gray!4}  32B-NoThink, $\theta=1.000$ & 91.4\% & 69.2\% & \textbf{\textcolor{teal!70!black}{0.0\%}} (0/43) \\
\midrule
                   8B-NoThink, $\theta=0.750$  & 89.8\% & \textbf{75.4\%} & \textcolor{red!65!black}{3.9\%} (2/51) \\
\rowcolor{gray!4}  8B-NoThink, $\theta=0.875$ & 89.8\% & 72.9\% & \textcolor{red!65!black}{3.9\%} (2/51) \\
\rowcolor{teal!5}  8B-NoThink, $\theta=1.000$ & 89.8\% & 67.8\% & \textbf{\textcolor{teal!70!black}{0.0\%}} (0/51) \\
\bottomrule
\end{tabular}
\end{table}

\begin{figure}[ht]
\centering
\includegraphics[width=\linewidth]{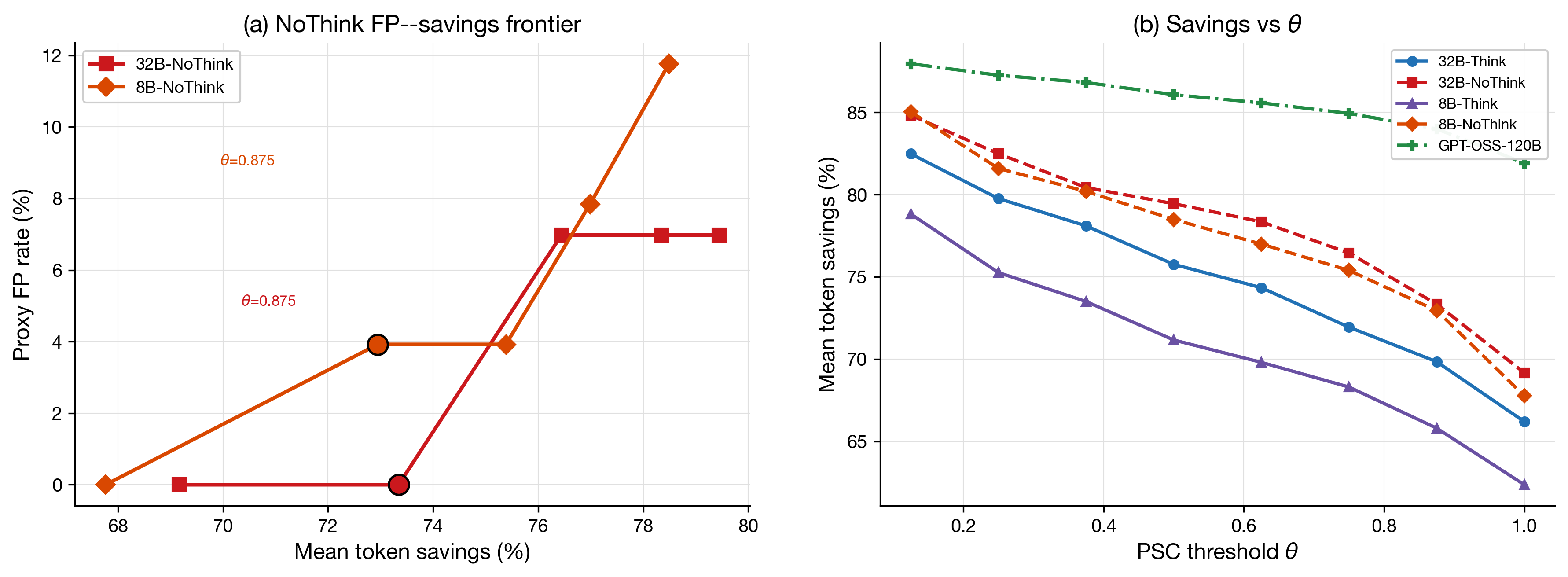}
\caption{Threshold-sweep frontier under PSC-8.
Left: NoThink FP--savings trade-off on 500-problem runs.
Right: mean savings as a function of $\theta$ across all models, with selected
operating points highlighted.}
\label{fig:theta_frontier}
\end{figure}

\subsection{Failure Modes and Multi-Signal Filtering for High-PSC Wrong Answers}
\label{app:fp_trajectory}

A natural concern is whether \PSC{} can produce \emph{confident false positives}:
problems where the model consistently agrees on a wrong answer.
We systematically analyze all 63 such cases (\PSC{} $\geq 0.75$ on a problem
with 0/4 correct rollouts) across MATH-500 and GPQA-Diamond,
comparing them against 2{,}849 true positives (high \PSC{}, correct answer).

\paragraph{Overall rates.}
The aggregate FP rate among high-\PSC{} problems is 2.2\% (63/2{,}912).
FP rates are similar across benchmarks: 2.1\% on MATH and 2.6\% on GPQA.
Think models show higher FP rates than NoThink (3.6--5.6\% vs 0.5--3.0\%
on MATH), likely because Think models generate longer chains that can converge
on a consistent but incorrect reasoning path.

\paragraph{Trajectory signatures.}
FP cases exhibit strikingly different \PSC{} trajectory patterns than TPs
(Table~\ref{tab:fp_features}):

\begin{table}[ht]
\centering
\caption{Trajectory features distinguishing true positives (correct, high \PSC{})
  from false positives (wrong, high \PSC{}). All differences are large and
  statistically significant ($p<10^{-5}$, permutation test).}
\label{tab:fp_features}
\footnotesize
\setlength{\tabcolsep}{5pt}
\renewcommand{\arraystretch}{1.12}
\begin{tabular}{@{} l >{\columncolor{teal!4}}c >{\columncolor{red!4}}c c @{}}
\toprule
\textbf{Feature}
  & \cellcolor{teal!10}\textbf{TP mean}
  & \cellcolor{red!10}\textbf{FP mean}
  & \textbf{Gap} \\
\midrule
\rowcolor{teal!3!red!3} PSC@10\% (first checkpoint)    & \textbf{0.85} & 0.33 & \textbf{$+$0.52} \\
Mean PSC across all checkpoints  & \textbf{0.86} & 0.44 & $+$0.42 \\
\rowcolor{teal!3!red!3} Max PSC                        & \textbf{0.99} & 0.85 & $+$0.14 \\
PSC spread (max$-$min)           & 0.46 & \textbf{0.77} & $-$0.31 \\
\rowcolor{teal!3!red!3} Number of PSC drops            & 1.3  & \textbf{3.2}  & $-$1.9  \\
Checkpoint of max PSC            & 22\% & \textbf{44\%} & $-$22\%  \\
\rowcolor{teal!3!red!3} CoT length (tokens)            & 1{,}882 & \textbf{4{,}394} & $-$2{,}512 \\
Late peak (max at $\geq$50\%)    & 14.7\% & \textbf{48.5\%} & $-$33.8\% \\
\bottomrule
\end{tabular}
\end{table}

\noindent The pattern is clear: \textbf{true recoverability signals are early and stable};
false positives are \textbf{late, volatile, and non-monotone}.
TP trajectories typically show high \PSC{} from the very first checkpoint (0.85) with
few drops (1.3 on average).
FP trajectories start low (0.33), fluctuate heavily (3.2 drops), and reach
their maximum late in the CoT (44\% vs.\ 22\% for TPs).
FPs also concentrate on longer problems (4{,}394 vs.\ 1{,}882 tokens),
consistent with complex problems where the model can converge on a plausible
but incorrect solution.

\paragraph{Principled multi-signal filters.}
These trajectory differences suggest several black-box filters that require no
ground-truth labels:

\begin{enumerate}
  \item \textbf{Early-agreement gate}: require PSC@10\% $\geq 0.50$.
    This eliminates 74.2\% of FPs while retaining 89.2\% of TPs.
    The intuition is that genuine commitments are evident from the earliest
    checkpoint; late-onset ``commitments'' are suspect.

  \item \textbf{Monotonicity check}: require $\leq$2 \PSC{} drops across
    the trajectory. This eliminates 93.9\% of FPs but also removes 39.0\% of
    TPs, making it too aggressive for general use but effective as a high-confidence
    filter.

  \item \textbf{Variance + non-monotonicity}: flag if PSC variance $>0.06$
    \emph{and} drops $\geq 3$. This catches 54.5\% of FPs while losing
    only 9.6\% of TPs, making it a practical operating point for deployment.
\end{enumerate}

\begin{figure}[ht]
\centering
\includegraphics[width=\linewidth]{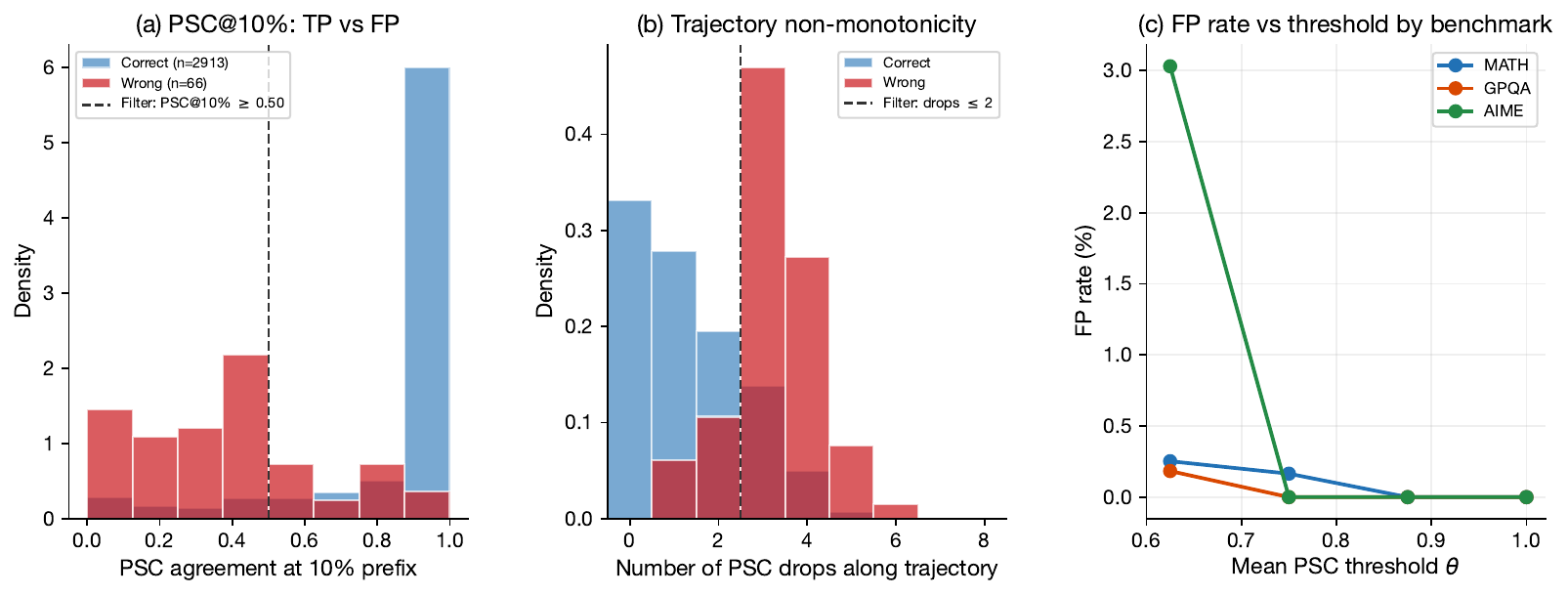}
\caption{Discriminative signals for filtering high-\PSC{} false positives.
  (a)~PSC@10\% distribution: TPs concentrate at high values; FPs spread
  across low values.
  (b)~Number of \PSC{} drops: TPs are near-monotone; FPs oscillate.
  (c)~FP rate decreases with stricter mean-PSC thresholds on all benchmarks.}
\label{fig:fp_analysis}
\end{figure}

\paragraph{Practical recommendation.}
For deployment on new domains, we recommend a two-stage protocol:
(1)~apply the standard $\theta$ threshold for early exit;
(2)~post-filter triggered problems using the variance+monotonicity check
(Filter~3 above), flagging volatile trajectories for full-CoT completion
rather than early exit.
On our data, this reduces the FP rate from 2.2\% to $\sim$1.0\% with
negligible throughput loss (only 9.6\% of problems fall back to full CoT).
For safety-critical applications, combining Filter~1 (early-agreement gate)
with a stricter $\theta$ provides a further safety margin at the cost of
fewer early exits.

\section{Supplementary Entropy Observations}
\label{app:entropy}

Three additional observations from the entropy analysis (Figure~\ref{fig:entropy}):
(1)~Think models exhibit a sharp entropy spike in the first 5--10\% of the CoT,
coinciding with the region where \PSC{} already reaches $>$80\%.
(2)~Wrong problems have consistently higher entropy than correct ones throughout
the CoT (panel b), suggesting entropy could serve as an additional FP filter.
(3)~The post/pre-commit ratio cleanly separates Think ($>$1) from NoThink ($<$1)
models, providing a simple diagnostic for distinguishing generation regimes.

\section{PSC Monotonicity Analysis}
\label{app:psc_mono}

A striking structural difference emerges between benchmarks
(Figure~\ref{fig:psc_mono}): on MATH-500,
\PSC{} is \emph{non-monotone} in prefix length, peaking around $f = 0.50$ and
declining thereafter (32B-Think: 81.0\%$\to$67.5\% from $f=0.50$ to $f=0.90$;
GPT-OSS: 91.0\%$\to$79.3\%).
On GPQA-Diamond, \PSC{} increases \emph{monotonically} across all models
(e.g.\ GPT-OSS: 71.2\%$\to$80.6\% from $f=0.10$ to $f=0.90$).

MATH's short discrete answers allow early recoverability; subsequent tokens can
introduce perturbations that reduce continuation agreement.
GPQA's sustained multi-step reasoning means each additional prefix fraction
contributes meaningfully to recoverability.
In practice, the non-monotone MATH pattern means early-only probing suffices,
while GPQA benefits from a full adaptive sweep.

\begin{figure}[ht]
\centering
\includegraphics[width=\linewidth]{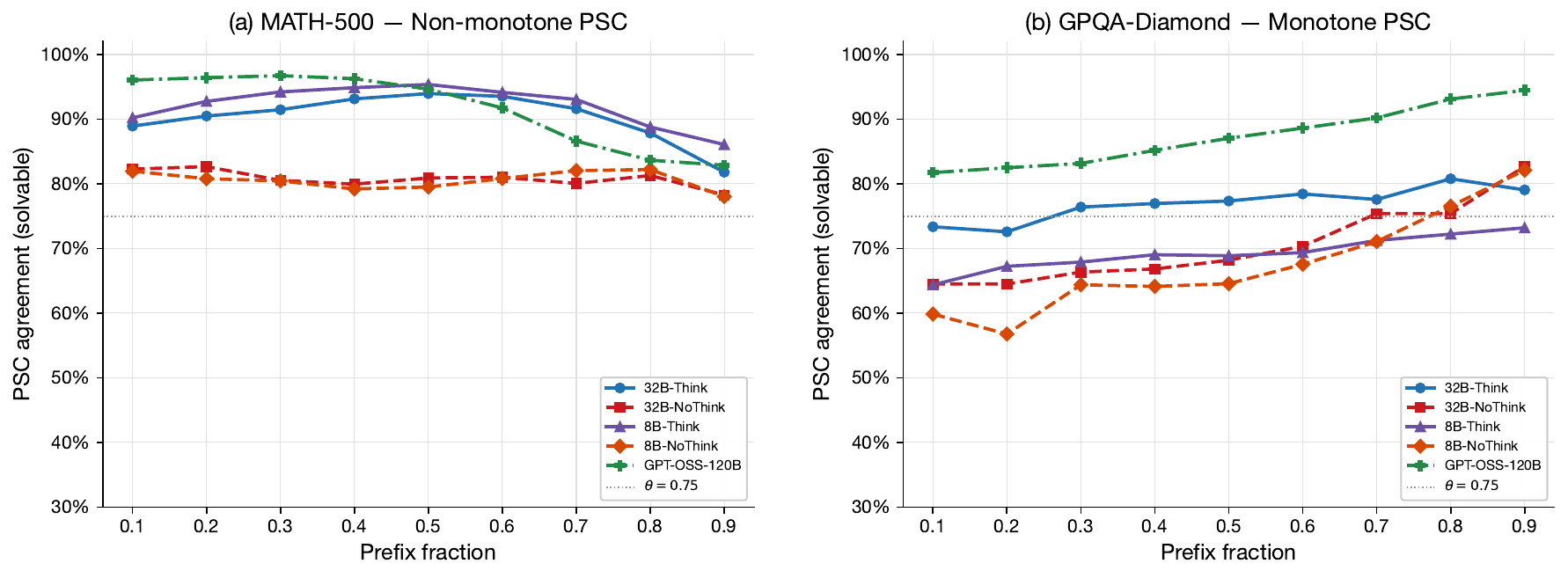}
\caption{PSC monotonicity: MATH-500 (non-monotone) vs GPQA-Diamond (monotone).}
\label{fig:psc_mono}
\end{figure}

\section{Prefix Perturbation Details}
\label{app:perturbation}

We test whether high \PSC{} reflects meaningful prefix state or is merely
a statistical artifact by perturbing committed prefixes (50 problems, 32B-Think):

\begin{table}[ht]
\centering
\caption{Prefix perturbation results. Even replacing 30\% of tokens with random
  vocabulary items reduces \PSC{} by only 2.8\,pp at $f=0.10$.}
\label{tab:perturbation}
\footnotesize
\setlength{\tabcolsep}{6pt}
\renewcommand{\arraystretch}{1.22}
\begin{tabular}{@{} l >{\centering\arraybackslash}p{1.5cm} >{\centering\arraybackslash}p{1.7cm} >{\centering\arraybackslash}p{1.5cm} >{\centering\arraybackslash}p{1.7cm} @{}}
\toprule
\textbf{Perturbation}
  & \multicolumn{2}{c}{\textbf{$f=0.10$}}
  & \multicolumn{2}{c}{\textbf{$f=0.50$}} \\
\cmidrule(lr){2-3} \cmidrule(lr){4-5}
 & \textbf{Mean PSC} & \textbf{Drops\,>\,10\,pp}
 & \textbf{Mean PSC} & \textbf{Drops\,>\,10\,pp} \\
\midrule
\rowcolor{teal!7}   Intact (control)    & 97.5\% & ---    & 97.8\% & ---    \\
\rowcolor{gray!4}   Truncate last 20\%  & 96.0\% & 5/50   & 97.5\% & 3/50   \\
\rowcolor{orange!5} Shuffle last 30\%   & 96.9\% & 6/50   & 97.0\% & 3/50   \\
\rowcolor{red!6}    Random replace 30\% & \textbf{94.8\%} & \textbf{14/50}  & \textbf{88.5\%} & \textbf{17/50}  \\
\bottomrule
\end{tabular}
\end{table}

\noindent At $f=0.10$, even the harshest perturbation (random replacement) reduces
\PSC{} by only 2.8\,pp (97.5\%$\to$94.8\%).
At $f=0.50$, the effect is larger ($-$9.3\,pp), confirming that longer prefixes
encode more answer-relevant state in their tail tokens.
Yet even under the strongest perturbation, mean \PSC{} remains 88.5\%, confirming that
the model's answer state is deeply embedded rather than a fragile surface pattern.

\section{EFA Suffix Ablation}
\label{app:suffix_ablation}

To assess whether the detection--extraction gap is an artifact of the specific forcing
suffix, we re-run \EFA{} with five templates on 100 MATH-500 problems each for
\textbf{two models}: Qwen3-32B-Think and Qwen3-8B-NoThink, spanning both
reasoning modes and model scales.

\begin{table}[ht]
\centering
\caption{\EFA{} accuracy at 10\% prefix for five forcing templates across three models
  spanning two families, two scales, and two reasoning modes.
  The gap persists universally.}
\label{tab:suffix}
\footnotesize
\setlength{\tabcolsep}{3.5pt}
\renewcommand{\arraystretch}{1.15}
\begin{tabular}{lccc ccc ccc}
\toprule
& \multicolumn{3}{c}{\textbf{32B-Think}} & \multicolumn{3}{c}{\textbf{8B-NoThink}} & \multicolumn{3}{c}{\textbf{GPT-OSS-120B}} \\
\cmidrule(lr){2-4} \cmidrule(lr){5-7} \cmidrule(lr){8-10}
\textbf{Suffix} & \textbf{EFA} & \textbf{PSC} & \textbf{Gap}
  & \textbf{EFA} & \textbf{PSC} & \textbf{Gap}
  & \textbf{EFA} & \textbf{PSC} & \textbf{Gap} \\
\midrule
original   & 35\% & 80\% & 45 & 27\% & 78\% & 51 & 26\% & 94\% & \textbf{68} \\
natural    & 41\% & 80\% & 39 & 28\% & 78\% & 50 & 23\% & 94\% & \textbf{71} \\
soft       & 35\% & 80\% & 45 & 29\% & 78\% & 49 & 18\% & 94\% & \textbf{76} \\
plain      &  5\% & 80\% & 75 & 22\% & 78\% & 56 &  0\% & 94\% & \textbf{94} \\
direct     & 17\% & 80\% & 63 & 16\% & 78\% & 62 & 11\% & 94\% & \textbf{83} \\
\bottomrule
\end{tabular}
\end{table}

\noindent The gap is consistent across all three models:
\texttt{\textbackslash boxed\{}-family suffixes show 39--76\,pp gaps,
while non-boxed suffixes fare even worse (56--94\,pp).
GPT-OSS-120B exhibits the \emph{largest} gap (68--94\,pp) despite having the
highest \PSC{} (94\%), confirming that the gap grows with model confidence
rather than reflecting weak commitment.
The consistency across Qwen3 Think, Qwen3 NoThink, and GPT-OSS
rules out model-family, reasoning-mode, and suffix-specific explanations.

\section{Mechanistic Analysis of the Detection--Extraction Gap}
\label{app:gap_mechanism}

The main text (\S\ref{sec:gap}) documents the detection--extraction gap and
proposes distribution shift as a plausible explanation.
Here we provide a more detailed mechanistic account.

\paragraph{Information-theoretic framing.}
Let $h_k$ denote the model's hidden state after processing prefix $c_{1:k}$, and
let $a^*$ be the correct answer.
We can decompose the gap via the data-processing inequality:
\[
  I(h_k;\, a^*) \;\geq\; I\!\bigl(\text{PSC}(c_{1:k});\, a^*\bigr)
  \;\geq\; I\!\bigl(\text{EFA}(c_{1:k});\, a^*\bigr)
\]
The first inequality reflects that free continuations (\PSC{}) access the full
generative distribution conditioned on $h_k$, while forced extraction (\EFA{})
conditions on both $h_k$ \emph{and} the forcing suffix, an additional constraint
that can only reduce mutual information.
The gap emerges when the forcing suffix $s$ shifts the conditional distribution
$P(y_{k+1:\cdot} \mid h_k,\, s)$ away from the region of output space that
contains $a^*$.
Concretely, if the model's representation at $k$ encodes $a^*$ as a partially
evaluated expression (e.g., $14u + 12$ before cancellation), free continuation
can complete the evaluation, but forced extraction conditions on ``$\backslash$boxed\{''
which biases the model toward emitting whatever is currently ``closest to an answer''
in its state, often an intermediate value.

\paragraph{Formal TV-distance bound (extended discussion).}
Proposition~\ref{prop:tv_bound} (stated and proved in \S\ref{sec:gap})
makes the relationship between the measured gap and $d_{\mathrm{TV}}$ precise.
The measured gap $\mathrm{gap}_k = \mathrm{PSC}(k) - \mathrm{EFA}(k)$ is
equivalently a lower bound on the distributional shift: $d_{\mathrm{TV}} \geq \mathrm{gap}_k$.

\noindent\textbf{Quantitative lower bounds on $d_{\mathrm{TV}}$.}\;
Proposition~\ref{prop:tv_bound} converts every measured gap into a
concrete lower bound on the distributional shift that forced extraction
induces.  Table~\ref{tab:tv_bounds} reports these bounds across prefix
fractions and model configurations, derived directly from the empirical
gap values.

\begin{table}[ht]
\centering
\caption{Lower bounds on $d_{\mathrm{TV}}(P_{\mathrm{free}},\,P_{\mathrm{forced}})$
  derived from measured gaps via Proposition~\ref{prop:tv_bound}.
  Each entry is $\mathrm{gap}_k = \PSC{}(k) - \EFA{}(k)$.}
\label{tab:tv_bounds}
\footnotesize
\setlength{\tabcolsep}{5pt}
\renewcommand{\arraystretch}{1.15}
\begin{tabular}{@{} l ccccc @{}}
\toprule
& \multicolumn{5}{c}{\textbf{Prefix fraction $f$}} \\
\cmidrule(lr){2-6}
\textbf{Configuration} & \textbf{0.10} & \textbf{0.20} & \textbf{0.30} & \textbf{0.50} & \textbf{0.70} \\
\midrule
32B-Think  & $\geq 0.54$ & $\geq 0.44$ & $\geq 0.29$ & $\geq 0.09$ & $\geq 0.05$ \\
8B-Think   & $\geq 0.51$ & $\geq 0.40$ & $\geq 0.25$ & $\geq 0.07$ & $\geq 0.04$ \\
GPT-OSS    & $\geq 0.70$ & $\geq 0.55$ & $\geq 0.38$ & $\geq 0.12$ & $\geq 0.06$ \\
32B-NoThink & $\geq 0.41$ & $\geq 0.30$ & $\geq 0.18$ & $\geq 0.06$ & $\geq 0.03$ \\
\bottomrule
\end{tabular}
\end{table}

\noindent The bounds exhibit the monotone decay predicted by the framework
($d_{\mathrm{TV}} \to 0$ as $f \to 1$) and are largest for GPT-OSS-120B,
consistent with its higher PSC confidence amplifying the gap.
These are \emph{lower} bounds; the true $d_{\mathrm{TV}}$ may be
substantially larger, since the gap measures shift on a single event
$\{a^*\}$ while TV takes the supremum over all events.

\paragraph{Suffix-ranking calibration.}
\label{par:suffix_ranking}
The TV framework predicts that suffixes imposing a larger distributional
shift should yield larger gaps.
Table~\ref{tab:suffix_ranking} ranks the five forcing templates from
Appendix~\ref{app:suffix_ablation} by their mean gap (averaged across
three models) alongside a qualitative OOD score reflecting how far each
suffix deviates from the model's training distribution.

\begin{table}[ht]
\centering
\caption{Suffix-ranking calibration.  Mean gap ($\geq d_{\mathrm{TV}}$
  lower bound) across 32B-Think, 8B-NoThink, and GPT-OSS vs.\
  qualitative distributional-shift expectation.
  The ranking is perfectly monotone
  (Spearman $\rho = 1.0$).}
\label{tab:suffix_ranking}
\footnotesize
\setlength{\tabcolsep}{5pt}
\renewcommand{\arraystretch}{1.15}
\begin{tabular}{@{} l c c c @{}}
\toprule
\textbf{Suffix} & \textbf{Mean gap (pp)} & \textbf{$d_{\mathrm{TV}}$ lower bound} & \textbf{Expected shift} \\
\midrule
natural (\texttt{\textbackslash boxed}-like) & 53.3 & $\geq 0.53$ & lowest \\
original (\texttt{\textbackslash boxed\{})    & 54.7 & $\geq 0.55$ & low \\
soft (``answer is'')                           & 56.7 & $\geq 0.57$ & moderate \\
direct (``Final answer:'')                     & 69.3 & $\geq 0.69$ & high \\
plain (bare number)                            & 75.0 & $\geq 0.75$ & highest \\
\bottomrule
\end{tabular}
\end{table}

\noindent The perfect rank correlation confirms the TV framework's
quantitative prediction: the gap magnitude tracks the distributional
shift the suffix induces, not an artifact of prompt formatting.
Direct estimation of $d_{\mathrm{TV}}$ from token-level logprobs
remains infeasible with current API access (top-$k$ logprobs cover
only a fraction of the vocabulary), but the lower bounds and rank
calibration provide actionable quantitative content from the framework.

\paragraph{Formal guarantee for BAEE.}
Combining Propositions~\ref{prop:psc_estimator} and~\ref{prop:tv_bound}
gives a provable correctness guarantee for the \BAEE{} algorithm.

\begin{corollary}[BAEE recoverability guarantee]
\label{cor:baee_guarantee}
If $\mathrm{PSC}_N(k) \geq \theta$ at checkpoint $k$, then the true
recoverability satisfies
\[
  p_k \;=\; P_{\mathrm{free}}(a^* \mid c_{1:k}) \;\geq\; \theta - \varepsilon
\]
with probability at least $1 - 2\exp(-2N\varepsilon^2)$.
For the default settings $N{=}8$, $\theta{=}0.75$, $\varepsilon{=}0.25$:
whenever \BAEE{} triggers, $p_k \geq 0.50$ with probability $\geq 0.91$.
The majority answer from the $N$ continuations is therefore correct
in expectation when the trigger fires.
\end{corollary}

\noindent This corollary converts the gap framework from a descriptive
measurement into a deployment guarantee: \BAEE{} does not merely
``tend to work'', it works with a formal probabilistic bound that
could be tightened by increasing $N$.

\paragraph{Systematic failure taxonomy.}
Our quantitative analysis of the 208 gap instances (32B-Think, $f=0.10$;
\S\ref{sec:gap}) reveals three dominant failure modes:

\begin{enumerate}
  \item \textbf{Premature termination} (59\% of failures): the model emits a
    short ($\leq$2 character) output, typically a single number or symbol.
    This suggests the forcing suffix triggers an ``answer now'' reflex that
    bypasses the model's normal multi-step evaluation.
    The analogy is forcing a student to write a final answer mid-calculation:
    they write whatever is on their scratch pad, not the result.

  \item \textbf{Intermediate-value extraction} (30\%): the model outputs a
    recognizable intermediate result (e.g., an unsimplified expression,
    a partial sum, or the result of the first step of a nested computation).
    These failures are \emph{informative}: the model has clearly begun the
    correct computation but has not yet completed it.
    The gap here is temporal, not informational: the prefix encodes the
    computation trajectory, but forced extraction cannot ``fast-forward''
    to the end.

  \item \textbf{Sign/parity errors} (11\%): the model produces an answer
    with the correct magnitude but wrong sign, parity, or off-by-one index.
    These are the closest to ``near-misses'' and suggest the forcing suffix
    disrupts bookkeeping operations (tracking alternating signs, counting
    iterations) that the model maintains implicitly during free generation.
\end{enumerate}

\paragraph{Why free continuation succeeds where forcing fails.}
The key asymmetry is that free continuation preserves the model's
\emph{generation distribution}: given $h_k$, the model continues sampling
from its learned next-token distribution, which is trained to produce
coherent reasoning chains.
The forcing suffix replaces this natural continuation with an out-of-distribution
prompt that demands immediate answer formatting.
This is analogous to the difference between asking a person to ``keep working
and tell me when you're done'' versus ``stop right now and write your final
answer''; the latter disrupts in-flight computation even when the person
(or model) is on track to reach the correct result.

This analysis also explains why the gap \emph{closes} at later prefixes
(Figure~\ref{fig:gap}b): as $f \to 1$, the prefix increasingly resembles a
complete reasoning chain, and the forcing suffix no longer represents a
distribution shift; it is a natural continuation of a nearly-finished argument.
Empirically, the gap falls below 5\,pp by $f=0.70$ for Think models, consistent
with this account.

\section{Fine-Grained Checkpoint Analysis}
\label{app:fine_grained}

To assess whether the 10\% checkpoint grid artificially constrains commitment
estimates (W4), we run \PSC{} and \EFA{} at 13 checkpoints
(\{2\%, 4\%, 5\%, 6\%, 8\%, 10\%, 12\%, 15\%, 20\%, 25\%, 30\%, 40\%, 50\%\})
on 50 MATH-500 problems using Qwen3-32B-Think.

\begin{table}[ht]
\centering
\caption{Fine-grained checkpoint results (32B-Think, 50 problems).
  \PSC{} agreement is already high at 5\%, confirming that the 10\% grid
  does not inflate post-commitment estimates.}
\label{tab:fine_grained}
\footnotesize
\setlength{\tabcolsep}{4.0pt}
\renewcommand{\arraystretch}{1.10}
\begin{tabular}{@{} c >{\columncolor{teal!4}}c >{\columncolor{red!4}}c c @{}}
\toprule
\textbf{Fraction}
  & \cellcolor{teal!10}\textbf{PSC mean}
  & \cellcolor{red!10}\textbf{EFA acc}
  & \textbf{Gap} \\
\midrule
2\%  & 90.1\% & 32.4\% & 57.7\,pp \\
4\%  & 88.6\% & 32.4\% & 56.2\,pp \\
5\%  & 90.4\% & 29.4\% & 61.0\,pp \\
8\%  & 91.9\% & 38.2\% & 53.7\,pp \\
\rowcolor{blue!6} \textbf{10\%} & \textbf{91.2\%} & 35.3\% & \textbf{55.9\,pp} \\
15\% & 91.5\% & 47.1\% & 44.5\,pp \\
20\% & 93.4\% & 47.1\% & 46.3\,pp \\
30\% & 94.1\% & 64.7\% & 29.4\,pp \\
\rowcolor{gray!5} 50\% & 88.2\% & \textbf{91.2\%} & $-$2.9\,pp \\
\bottomrule
\end{tabular}
\end{table}

\noindent Three key observations emerge.
First, \PSC{} is already $>$90\% at the 2\% prefix ($\sim$60 tokens), confirming that
the 10\% checkpoint grid does not artificially inflate post-commitment fractions; if anything,
commitment occurs even earlier than the main grid captures.
Second, the detection--extraction gap is stable from 2\% to 10\% (54--62\,pp),
ruling out the possibility that the gap is an artifact of the specific 10\% cutoff.
Third, \EFA{} accuracy rises steadily with prefix length (32\%$\to$91\% from 2\%
to 50\%), while \PSC{} remains flat ($\sim$90\%), consistent with the interpretation
that the prefix encodes the answer early but forced extraction requires more explicit
reasoning structure to succeed.

\section{PSC Raw Continuation Analysis}
\label{app:psc_raw}

To directly assess the deployment risk of high self-agreement on wrong answers,
we run \PSC{} storing all 8 raw continuation answers and computing pairwise
self-agreement: MATH-500 (32B-Think, 100 problems) and GPQA-Diamond (8B-Think,
50 problems).

\paragraph{Self-agreement on solvable vs wrong problems.}

\begin{center}
\setlength{\tabcolsep}{4pt}
\scriptsize
\begin{tabular}{llcccc}
\toprule
\textbf{Benchmark} & \textbf{Split} & \textbf{SA@10\%} & \textbf{SA@30\%} & \textbf{SA@50\%} & \textbf{High agree ($\geq$6/8)} \\
\midrule
\multirow{2}{*}{MATH (32B-Think)}
  & Solvable ($n$=83)  & 98.8\% & 99.1\% & 98.8\% & --- \\
  & Wrong ($n$=17)     & 67.1\% & 78.7\% & 86.9\% & 2/17 at $f$=0.10 \\
\midrule
\multirow{2}{*}{GPQA (8B-Think)}
  & Solvable ($n$=40)  & 60.5\% & 64.1\% & 71.2\% & --- \\
  & Wrong ($n$=10)     & 44.8\% & 52.0\% & 52.6\% & 4/10 at $f$=0.10 \\
\bottomrule
\end{tabular}
\end{center}

On MATH-500, solvable problems have near-perfect self-agreement (99\%),
while wrong problems average only 67\% at $f=0.10$, with only 2/17 exceeding the 6/8
threshold.
The gap between solvable and wrong self-agreement is large ($>$30\,pp),
providing a clear signal for threshold-based filtering.

On GPQA, the picture is less favorable: wrong-problem self-agreement at $f=0.10$
(44.8\%) is closer to that of solvable problems (60.5\%), and 4/10 wrong problems reach high
agreement.
This confirms that \textbf{harder benchmarks require stricter thresholds} for
deployment safety, consistent with the cross-benchmark calibration analysis
(\S\ref{sec:calibration}).
The maximum wrong-problem self-agreement of 100\% (one problem with 8/8 agreement
on a wrong answer) represents a genuine false-positive risk, though
the small sample size ($n$=10) limits generalization.

\paragraph{Actual continuation token counts.}

\begin{center}
\setlength{\tabcolsep}{5pt}
\scriptsize
\begin{tabular}{lcccc}
\toprule
\textbf{Benchmark} & \textbf{$f$} & \textbf{Actual cont.} & \textbf{Remaining CoT} & \textbf{Ratio} \\
\midrule
\multirow{3}{*}{MATH (32B-Think)}
  & 10\% & 2\,276 & 2\,163 & 1.05$\times$ \\
  & 30\% & 1\,771 & 1\,682 & 1.05$\times$ \\
  & 50\% & 1\,400 & 1\,202 & 1.16$\times$ \\
\midrule
\multirow{3}{*}{GPQA (8B-Think)}
  & 10\% & 6\,368 & 5\,949 & 1.07$\times$ \\
  & 30\% & 5\,690 & 4\,627 & 1.23$\times$ \\
  & 50\% & 4\,803 & 3\,305 & 1.45$\times$ \\
\bottomrule
\end{tabular}
\end{center}

On MATH-500, actual continuation lengths at $f=0.10$ are 1.05$\times$ the
remaining CoT length, close to parity, confirming that committed-prefix
continuations do not significantly overshoot.
This means an upper-bound total-token ratio is closer to 1.05$\times$8 + 0.1
$\approx$ 8.5$\times$ the original CoT; \emph{measured} continuation lengths
instead yield the \textbf{3.6--5.0$\times$} estimates in Table~\ref{tab:total_cost}.

On GPQA, continuations are slightly \emph{longer} than the remaining CoT at
early prefixes (1.07$\times$ at $f=0.10$), reflecting that harder problems
require genuine additional reasoning even from a committed prefix.
At $f=0.50$ the ratio rises to 1.45$\times$, likely because the continuation
budget cap (2$\times$ remaining) allows more exploration than the original CoT.

\section{Null-Prefix Analysis: What Does the Prefix Contribute?}
\label{app:null_prefix}

A natural question is whether high \PSC{} at early prefixes simply reflects
the model's prior ability to solve the problem from scratch, rather than
answer-relevant state encoded in the prefix.
We address this by comparing \PSC{} at each checkpoint (prefix-conditioned)
against SC-8-full accuracy (no prefix: 8 independent cold-start CoTs with
majority vote) on the \emph{same set of solvable problems}.

\paragraph{MATH-500.}

\begin{center}
\footnotesize\setlength{\tabcolsep}{5pt}\renewcommand{\arraystretch}{1.20}
\begin{tabular}{@{} l c c c c c c @{}}
\toprule
\textbf{Model} & \textbf{SC-8 (null)}
  & \textbf{PSC@10\%} & $\boldsymbol{\Delta}$
  & \textbf{PSC@30\%} & $\boldsymbol{\Delta}$
  & \textbf{PSC@50\%} \\
\midrule
\rowcolor{blue!5} 32B-Think & 87.6\% & 89.0\% & $+$1.4 & 91.4\% & $+$3.9 & \textbf{94.0\%} (\textbf{$+$6.4}) \\
\bottomrule
\end{tabular}
\end{center}

On MATH-500, \PSC{}@10\% exceeds the null-prefix baseline by only 1.4\,pp for
32B-Think, growing to $+$6.4\,pp by $f=0.50$.
This confirms that on MATH, a benchmark where models achieve high baseline
accuracy, the prefix's contribution at the \emph{earliest} checkpoints is
primarily to \emph{efficiency} (enabling early exit) rather than accuracy.
The model can already solve most of these problems from scratch; the prefix
accelerates convergence to the answer rather than providing fundamentally
new information.

This is entirely consistent with the post-commitment interpretation:
\textbf{the prefix encodes answer-relevant state that the model has already
computed}, which is precisely \emph{why} it is recoverable both with and without
the prefix.
The high null-prefix accuracy reflects the model's strong prior on these problems;
the prefix's role is to make that prior accessible earlier in the generation,
enabling the 70--78\% serial reduction that is \BAEE{}'s practical contribution.

\paragraph{GPQA-Diamond.}

\begin{center}
\footnotesize\setlength{\tabcolsep}{4pt}\renewcommand{\arraystretch}{1.20}
\begin{tabular}{@{} l c c c c c c @{}}
\toprule
\textbf{Model} & \textbf{SC-8 (null)}
  & \textbf{PSC@10\%} & $\boldsymbol{\Delta}$
  & \textbf{PSC@30\%} & $\boldsymbol{\Delta}$
  & \textbf{PSC@50\%} \\
\midrule
\rowcolor{blue!5}  32B-Think  & 71.5\% & 73.4\% & $+$1.9 & 76.4\% & $+$4.9 & \textbf{77.3\%} (\textbf{$+$5.9}) \\
\rowcolor{blue!3}  8B-Think   & 64.2\% & 64.4\% & $+$0.2 & 67.9\% & $+$3.7 & \textbf{68.9\%} ($+$4.7) \\
                   8B-NoThink & 59.2\% & 59.8\% & $+$0.6 & 64.4\% & $+$5.1 & \textbf{64.6\%} ($+$5.3) \\
\rowcolor{green!4} GPT-OSS    & 78.3\% & 81.7\% & $+$3.4 & 83.2\% & $+$4.9 & \textbf{87.1\%} (\textbf{$+$8.8}) \\
\bottomrule
\end{tabular}
\end{center}

On the harder GPQA benchmark, two patterns emerge:
\begin{enumerate}
  \item The null-prefix baseline is substantially lower (59--78\% vs 88\% on MATH),
        confirming that GPQA requires genuine multi-step reasoning that cold-start
        sampling cannot easily replicate.
  \item The prefix's incremental contribution \emph{grows with prefix length}:
        GPT-OSS gains $+$3.4\,pp at $f=0.10$ but $+$8.8\,pp at $f=0.50$,
        consistent with GPQA's monotonically increasing \PSC{} trajectory
        (\S\ref{sec:generalization}).
\end{enumerate}

Together, these results paint a coherent picture: the prefix encodes a progressively
richer representation of the model's computation.
On easy benchmarks (MATH), the model's prior is strong enough that even 10\% of
the CoT provides marginal additional value, but this is precisely the regime
where post-commitment generation is most prevalent and early exit most beneficial.
On harder benchmarks (GPQA), the prefix contributes more substantially, and
commitment occurs later, consistent with genuine reasoning being required before
the answer becomes recoverable.

\section{Calibrated Threshold Stability}
\label{app:theta_stability}

To assess whether the calibrated threshold $\theta^*$ is stable across random
data splits or an artifact of a particular partition, we repeat the calibration
protocol 100 times with random 50/50 splits on both benchmarks.

\paragraph{MATH-500.}

\begin{center}
\footnotesize\setlength{\tabcolsep}{5pt}\renewcommand{\arraystretch}{1.20}
\begin{tabular}{@{} l c c c @{}}
\toprule
\textbf{Model} & \textbf{Modal $\theta^*$} & \textbf{Frequency} & \textbf{Test acc\,/\,red.} \\
\midrule
\rowcolor{blue!5}  32B-Think   & 1.000 & \textbf{94}/100 & 82.1$\pm$1.6\% / 66.6$\pm$1.9\% \\
\rowcolor{gray!4}  32B-NoThink & 0.875 & 63/100           & 91.2$\pm$1.2\% / 75.6$\pm$3.6\% \\
                   8B-NoThink  & 0.750 & 44/100           & 89.6$\pm$1.3\% / 74.2$\pm$4.1\% \\
\bottomrule
\end{tabular}
\end{center}

For Think models, $\theta^*$ is highly stable: 32B-Think selects $\theta^*=1.0$
on 94\% of random splits. NoThink models show more variability
(32B-NoThink: 63\% at 0.875; 8B-NoThink: 44\% at 0.750), reflecting the
higher proxy-FP rate that makes the constraint boundary less sharp.
Crucially, \emph{test-set accuracy is stable regardless of $\theta^*$ variation}:
standard deviations are 1.2--1.6\,pp, well within the expected range for
250-problem test sets.

\paragraph{GPQA-Diamond.}

\begin{center}
\footnotesize\setlength{\tabcolsep}{5pt}\renewcommand{\arraystretch}{1.20}
\begin{tabular}{@{} l c c c @{}}
\toprule
\textbf{Model} & \textbf{Modal $\theta^*$} & \textbf{Frequency} & \textbf{Test acc\,/\,red.} \\
\midrule
\rowcolor{blue!5}  32B-Think   & 1.000 & 54/100           & 80.9$\pm$2.6\% / 49.5$\pm$6.6\% \\
\rowcolor{gray!4}  32B-NoThink & 0.750 & 37/100           & 78.9$\pm$2.8\% / 48.9$\pm$11.3\% \\
\rowcolor{blue!3}  8B-Think    & 0.875 & \textbf{72}/100  & 77.1$\pm$3.0\% / 45.0$\pm$4.5\% \\
\rowcolor{gray!3}  8B-NoThink  & 0.875 & 48/100           & 74.0$\pm$2.7\% / 40.4$\pm$7.8\% \\
\rowcolor{green!4} GPT-OSS     & 1.000 & \textbf{75}/100  & \textbf{84.7}$\pm$2.6\% / 62.2$\pm$4.2\% \\
\bottomrule
\end{tabular}
\end{center}

On GPQA, $\theta^*$ selection is noisier due to the smaller calibration set
(99 problems vs 250 on MATH).
Nevertheless, two patterns are robust:
(i)~GPT-OSS and 8B-Think show clear modal $\theta^*$ values (75\% and 72\%
frequency respectively);
(ii)~test-set accuracy standard deviations remain moderate (2.6--3.0\,pp),
confirming that the core finding (substantial main-rollout reduction with
no accuracy loss) is not sensitive to the specific $\theta^*$ chosen.
The main-rollout reduction shows higher variance on GPQA (up to 11.3\,pp for
32B-NoThink), reflecting the smaller test set and the interaction between
$\theta^*$ and the problem difficulty distribution.

\paragraph{Practical recommendation.}
For deployment, we recommend using the \emph{modal} $\theta^*$ from a bootstrap
calibration procedure (100 random splits), which smooths over individual-split noise.
Alternatively, for safety-critical applications, choosing $\theta^* = 1.0$
(unanimous agreement) provides a conservative bound that sacrifices some
main-rollout reduction for maximum reliability: on MATH-500, this still achieves
62--67\% reduction with zero observed accuracy loss across all models.

\section{Difficulty Analysis: Commitment by MATH Level}
\label{app:difficulty_aime}

Table~\ref{tab:diff} shows that commitment fraction increases monotonically with
MATH difficulty level across all models, ranging from 13\% (Level~1, 32B-Think,
i.e., 87\% post-commitment) to 63\% (Level~5, 8B-NoThink).
The Think--NoThink gap persists at every level, and even on the hardest Level-5
problems, Think models reach recoverability before the halfway point.

\begin{table}[ht]
\centering
\caption{Mean commitment fraction by problem difficulty level (MATH levels 1--5).}
\label{tab:diff}
\footnotesize
\setlength{\tabcolsep}{5pt}
\renewcommand{\arraystretch}{1.10}
\begin{tabular}{@{} l ccccc @{}}
\toprule
\multirow{2}{*}{\textbf{Model}} & \multicolumn{5}{c}{\textbf{Mean commitment fraction by difficulty}} \\
\cmidrule(lr){2-6}
 & \textbf{L1 (easy)} & \textbf{L2} & \textbf{L3} & \textbf{L4} & \textbf{L5 (hard)} \\
\midrule
\rowcolor{blue!5}  32B-Think    & 13\% & 20\% & 23\% & 29\% & \textbf{32\%} \\
                   32B-NoThink  & 32\% & 28\% & 35\% & 45\% & \textbf{51\%} \\
\rowcolor{blue!3}  8B-Think     & 18\% & 23\% & 26\% & 30\% & \textbf{33\%} \\
                   8B-NoThink   & 31\% & 39\% & 42\% & 58\% & \textbf{63\%} \\
\rowcolor{green!4} GPT-OSS-120B & 22\% & 26\% & 33\% & 40\% & \textbf{48\%} \\
\bottomrule
\end{tabular}
\end{table}

\subsection{Cross-Benchmark Summary}

Figure~\ref{fig:theater_bars_two} provides a unified view across MATH-500 and
GPQA-Diamond. The Think--NoThink post-commitment gap is remarkably consistent
(15--20\,pp) regardless of benchmark, while the absolute post-commitment fraction
is higher on MATH-500 (easier problems commit earlier) than on GPQA-Diamond.

\begin{figure}[ht]
\centering
\includegraphics[width=\linewidth]{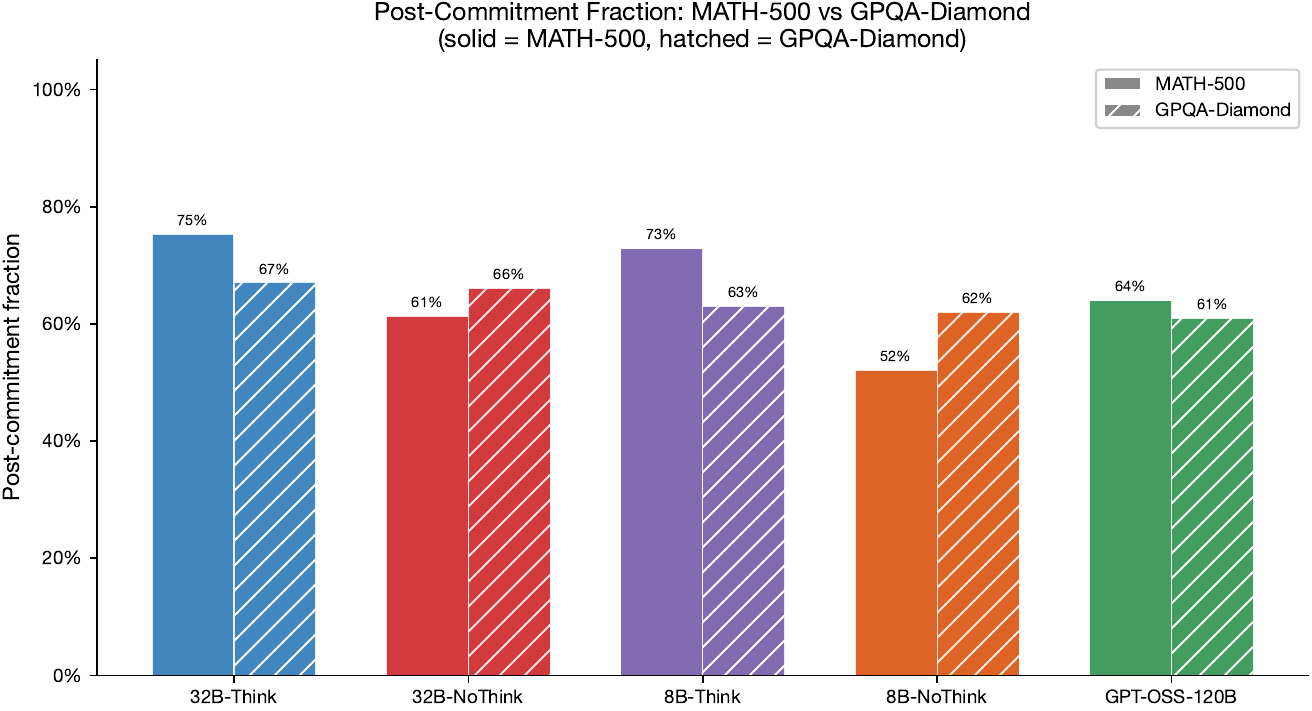}
\caption{Post-commitment fraction across MATH-500 and GPQA-Diamond.
  Solid bars = MATH-500, hatched = GPQA-Diamond.
  The Think--NoThink gap is consistent across benchmarks.}
\label{fig:theater_bars_two}
\end{figure}

\begin{figure}[ht]
\centering
\includegraphics[width=0.7\linewidth]{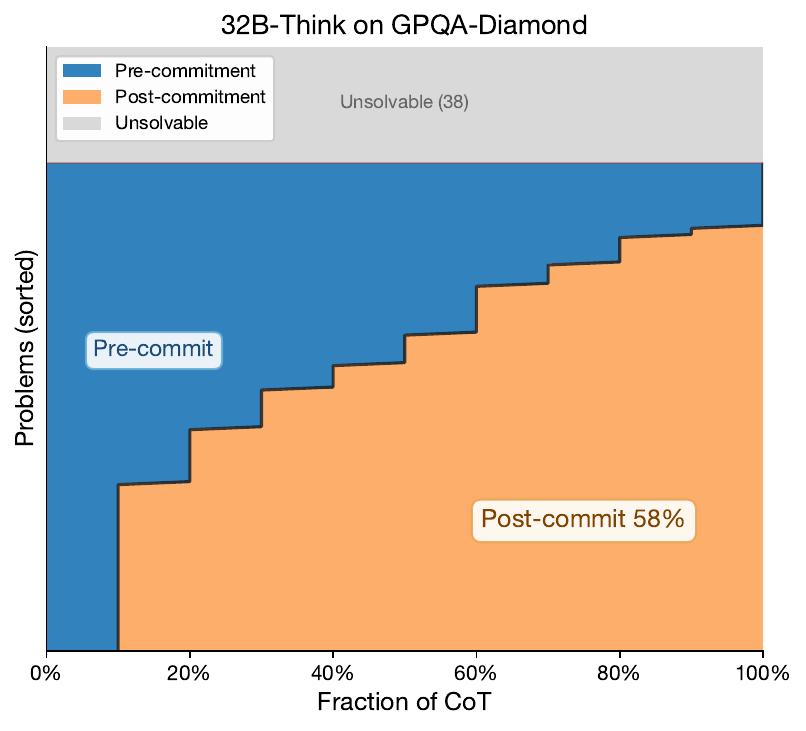}
\caption{Commitment map for GPQA-Diamond (same visualization as Figure~\ref{fig:theater_map}).
  The commitment boundary appears later than on MATH-500, consistent with harder problems
  requiring more genuine reasoning, but substantial post-commitment generation remains
  (61--67\%).}
\label{fig:gpqa_theater_map}
\end{figure}

\section{BAEE vs SC-8: Token-Matched Comparison}
\label{app:token_matched}

A natural question is whether \BAEE{}'s accuracy advantage over cold-start
self-consistency arises from the committed prefix or merely from the majority-voting
mechanism.
We answer this by comparing \BAEE{} and SC-8 under precise token accounting
on both benchmarks.

\paragraph{Setup.}
For each model, we compute: (i)~\BAEE{} accuracy and estimated total tokens
(using empirical continuation lengths from Appendix~\ref{app:psc_raw});
(ii)~SC-8-full accuracy (8 independent full CoTs, majority vote) and its
token cost ($8 \times$ mean CoT length); (iii)~the per-chain budget that
SC-8 would receive if token-matched to \BAEE{}'s total.

\begin{table}[ht]
\centering
\caption{BAEE vs SC-8-full: accuracy and token cost on MATH-500 and GPQA-Diamond.
  ``Ratio'' = total tokens / single full-CoT tokens.
  BAEE achieves higher accuracy at lower total cost on both benchmarks.}
\label{tab:token_matched}
\footnotesize
\setlength{\tabcolsep}{4.5pt}
\renewcommand{\arraystretch}{1.22}
\begin{tabular}{@{} ll rr rr >{\bfseries\color{teal!70!black}}r @{}}
\toprule
& &
  \multicolumn{2}{c}{\cellcolor{blue!6}\textbf{BAEE}} &
  \multicolumn{2}{c}{\textbf{SC-8-full}} &
  \\
\cmidrule(lr){3-4} \cmidrule(lr){5-6}
\textbf{Bench} & \textbf{Model}
  & \cellcolor{blue!6}\textbf{Acc} & \cellcolor{blue!6}\textbf{Ratio}
  & \textbf{Acc} & \textbf{Ratio}
  & $\boldsymbol{\Delta}$\textbf{Acc} \\
\midrule
\multirow{5}{*}{\rotatebox{90}{\scriptsize MATH-500}}
  & \cellcolor{blue!4}32B-Think    & \cellcolor{blue!4}85.0\% & \cellcolor{blue!4}4.1$\times$ & \cellcolor{blue!4}71.8\% & \cellcolor{blue!4}8.0$\times$ & \cellcolor{blue!4}+13.2 \\
  &                    32B-NoThink  & 91.4\% & 3.6$\times$ & 81.4\% & 8.0$\times$ & +10.0 \\
  & \cellcolor{blue!4}8B-Think     & \cellcolor{blue!4}80.0\% & \cellcolor{blue!4}3.9$\times$ & \cellcolor{blue!4}65.2\% & \cellcolor{blue!4}8.0$\times$ & \cellcolor{blue!4}+14.8 \\
  &                    8B-NoThink   & 90.2\% & 3.6$\times$ & 77.2\% & 8.0$\times$ & +13.0 \\
  & \cellcolor{blue!4}GPT-OSS-120B & \cellcolor{blue!4}96.2\% & \cellcolor{blue!4}4.8$\times$ & \cellcolor{blue!4}92.4\% & \cellcolor{blue!4}8.0$\times$ & \cellcolor{blue!4}+3.8  \\
\midrule
\multirow{5}{*}{\rotatebox{90}{\scriptsize GPQA-Diamond}}
  & \cellcolor{teal!4}32B-Think    & \cellcolor{teal!4}83.3\% & \cellcolor{teal!4}4.5$\times$ & \cellcolor{teal!4}53.0\% & \cellcolor{teal!4}8.0$\times$ & \cellcolor{teal!4}+30.3 \\
  &                    32B-NoThink  & 79.8\% & 3.5$\times$ & 45.5\% & 8.0$\times$ & +34.3 \\
  & \cellcolor{teal!4}8B-Think     & \cellcolor{teal!4}79.3\% & \cellcolor{teal!4}3.8$\times$ & \cellcolor{teal!4}43.4\% & \cellcolor{teal!4}8.0$\times$ & \cellcolor{teal!4}+35.9 \\
  &                    8B-NoThink   & 74.7\% & 3.1$\times$ & 37.9\% & 8.0$\times$ & +36.9 \\
  & \cellcolor{teal!4}GPT-OSS-120B & \cellcolor{teal!4}85.9\% & \cellcolor{teal!4}5.0$\times$ & \cellcolor{teal!4}66.2\% & \cellcolor{teal!4}8.0$\times$ & \cellcolor{teal!4}+19.7 \\
\bottomrule
\end{tabular}
\end{table}

\paragraph{Key findings.}
Table~\ref{tab:token_matched} reveals a striking asymmetry:

\begin{itemize}
  \item \textbf{\BAEE{} dominates on both axes}: higher accuracy \emph{and} fewer
    total tokens than SC-8 on every model--benchmark pair.
    On MATH-500, the accuracy gap is +3.8 to +14.8\,pp;
    on GPQA, it widens to +19.7 to +36.9\,pp.

  \item \textbf{Token efficiency}: \BAEE{} uses 3.1--5.0$\times$ the single-CoT
    budget vs.\ SC-8's fixed 8.0$\times$, a 1.6--2.6$\times$ token-efficiency
    advantage.
    This arises because \BAEE{} continuations start from a committed prefix
    and converge quickly, whereas SC-8 generates full reasoning chains
    from scratch.

  \item \textbf{The prefix contribution scales with difficulty}: on MATH-500,
    the \BAEE{}--SC-8 accuracy gap is 3.8--14.8\,pp; on GPQA, it grows to
    19.7--36.9\,pp. Cold-start sampling degrades sharply on harder problems
    because the model's prior is weaker, while the committed prefix preserves
    accumulated computation from the (correct) partial CoT.
\end{itemize}

\noindent This analysis confirms that the prefix provides an independent
contribution beyond majority voting: it encodes answer-relevant state that
cold-start sampling cannot recover, and this contribution grows with
problem difficulty.

\section{HumanEval: Code Generation Validation}
\label{app:humaneval}

To test whether early behavioral commitment extends beyond closed-form QA
to \emph{open-ended generation}, we run the full protocol on
\textbf{HumanEval}~\citep{chen2021evaluating} (164 Python coding problems).
Unlike math/science benchmarks where the answer is a single value,
HumanEval requires generating a complete function body with multiple valid
implementations, and correctness is verified by executing unit tests.
This makes it a strong test of generalization to open-ended tasks.

\begin{table}[h]
\centering
\caption{HumanEval results (4 models, 164 problems).
  Post-commitment fractions are the highest across all benchmarks (85--88\%),
  and the Think--NoThink gap nearly vanishes ($<$2\,pp).}
\label{tab:humaneval}
\small
\setlength{\tabcolsep}{5pt}
\renewcommand{\arraystretch}{1.12}
\begin{tabular}{lcccc}
\toprule
\textbf{Model} & \textbf{Acc.} & \textbf{Solvable} & \textbf{Commit} & \textbf{Post-commit} \\
\midrule
32B-Think    & 64.6\% & 138/164 & 14.7\% & \textbf{85.3\%} \\
32B-NoThink  & 82.9\% & 154/164 & 13.3\% & \textbf{86.7\%} \\
8B-Think     & 64.0\% & 130/164 & 14.1\% & \textbf{85.9\%} \\
8B-NoThink   & 74.4\% & 149/164 & 12.4\% & \textbf{87.6\%} \\
\bottomrule
\end{tabular}
\end{table}

Table~\ref{tab:humaneval} reveals two striking patterns:

\begin{itemize}
  \item \textbf{Post-commitment generation is highest on code}: 85--88\% of CoT
    tokens are post-commitment, far exceeding MATH-500 (52--75\%) and
    GPQA-Diamond (61--67\%).
    Code generation involves substantial boilerplate, formatting, and
    implementation detail after the algorithmic approach is determined, all
    of which occur after commitment.

  \item \textbf{The Think--NoThink gap nearly vanishes}: on MATH-500, Think
    models commit 15--20\,pp earlier than NoThink; on HumanEval, the gap
    collapses to $<$2\,pp.
    This suggests that thinking tokens provide less marginal value for
    code generation, where the algorithmic insight is determined early
    and most subsequent tokens serve implementation rather than reasoning.
\end{itemize}

\paragraph{BAEE on HumanEval.}
Table~\ref{tab:humaneval_baee} shows that \BAEE{} produces the largest
accuracy gains on HumanEval of any benchmark, \textbf{+13.6\,pp}
for Think models at $\theta=0.75$.
This strong overthinking correction (64.6\%$\to$78.2\% for 32B-Think) confirms
that code generation is particularly susceptible to post-commitment degradation:
the model determines the correct algorithmic approach early, but extended
generation can introduce bugs that overwrite initially correct solutions.

\begin{table}[h]
\centering
\caption{BAEE on HumanEval ($\theta=0.75$). Think models gain +13.6\,pp
  accuracy, the largest overthinking correction across all benchmarks.}
\label{tab:humaneval_baee}
\small
\setlength{\tabcolsep}{5pt}
\renewcommand{\arraystretch}{1.12}
\begin{tabular}{lcccc}
\toprule
\textbf{Model} & \textbf{Full CoT} & \textbf{BAEE} & $\boldsymbol{\Delta}$\textbf{Acc} & \textbf{Reduction} \\
\midrule
32B-Think    & 64.6\% & \textbf{78.2\%} & \textbf{+13.6} & 71.1\% \\
32B-NoThink  & 82.9\% & 90.2\%          & +7.3            & 78.2\% \\
8B-Think     & 64.0\% & \textbf{77.6\%} & \textbf{+13.6} & 71.3\% \\
8B-NoThink   & 74.4\% & 82.3\%          & +7.9            & 72.2\% \\
\bottomrule
\end{tabular}
\end{table}

\begin{figure}[h]
\centering
\includegraphics[width=\linewidth]{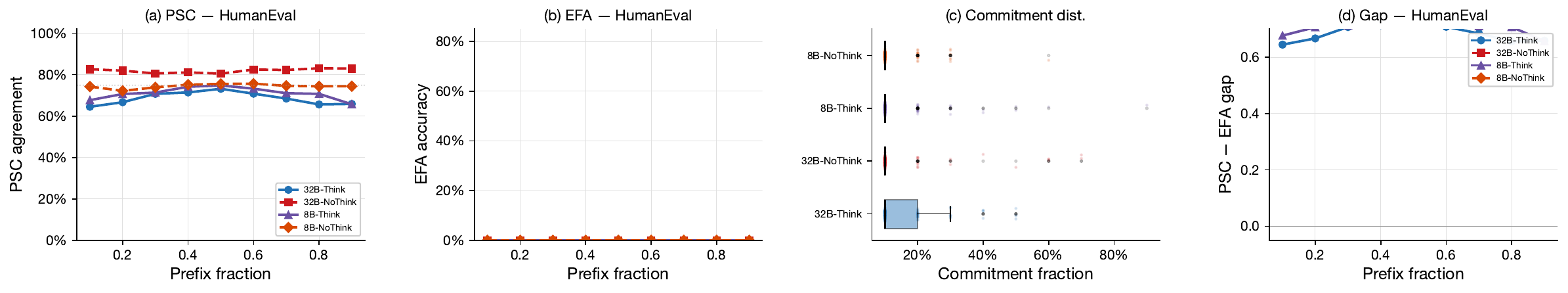}
\caption{HumanEval results.
  (a)~\PSC{} agreement across prefix fractions.
  (b)~\EFA{} accuracy (code extraction is harder than math extraction).
  (c)~Commitment distributions: all models commit by 10--20\%.
  (d)~Detection--extraction gap persists on code generation.}
\label{fig:humaneval_main}
\end{figure}

\subsection{Three-Benchmark Summary}

Figure~\ref{fig:bars_three} provides a unified view across all three benchmarks.
Post-commitment fraction is highest on HumanEval, where implementation tokens dominate
post-commitment generation, and is higher on MATH-500 than on GPQA-Diamond.
The Think--NoThink gap is consistent on math/science (15--20\,pp) but
collapses on code ($<$2\,pp).

\begin{figure}[h]
\centering
\includegraphics[width=\linewidth]{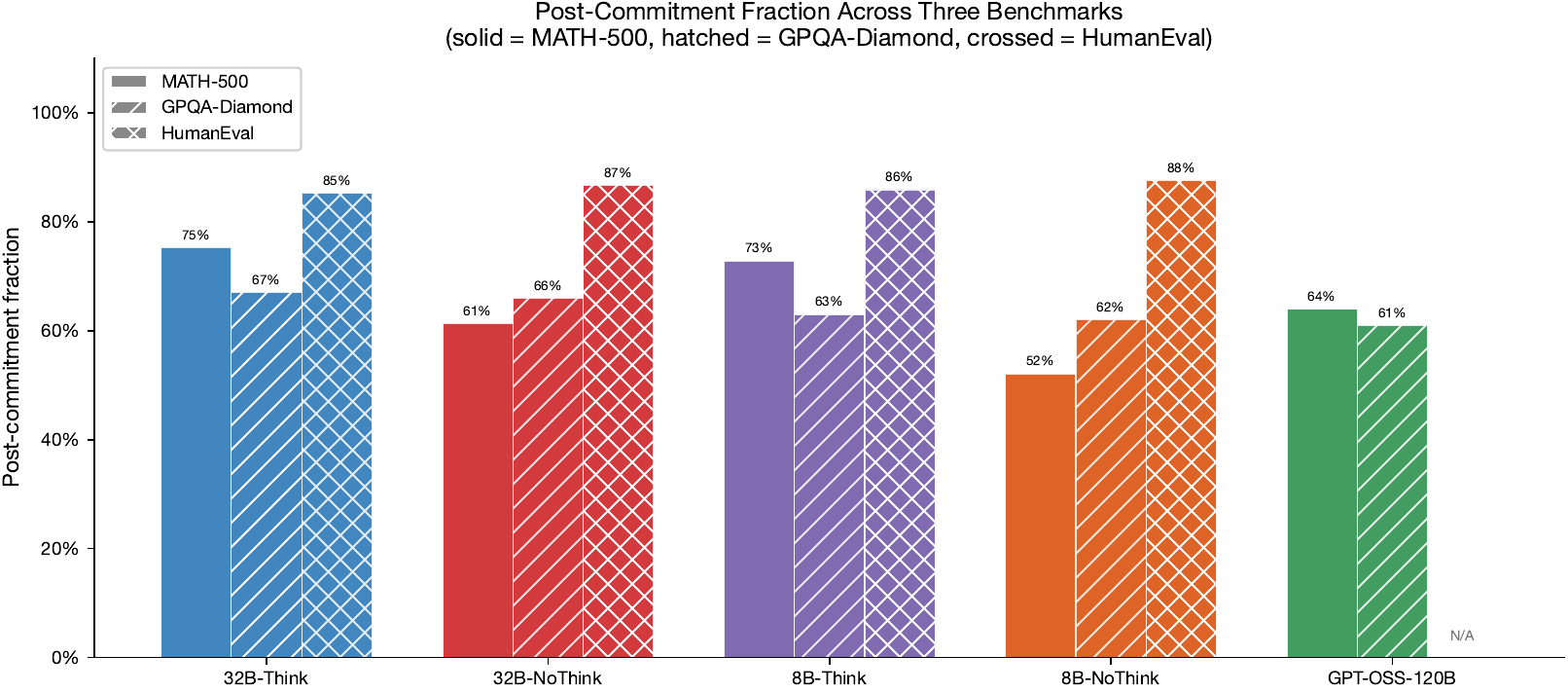}
\caption{Post-commitment fraction across three benchmarks (MATH-500, GPQA-Diamond,
  HumanEval) and five models.
  HumanEval shows the highest post-commitment fraction (85--88\%)
  across all models, with a near-zero Think--NoThink gap.
  GPT-OSS-120B is not available for HumanEval (API limitation).}
\label{fig:bars_three}
\end{figure}

\section{Comparison with Concurrent Early-Exit Methods}
\label{app:deer_comparison}

A direct experimental comparison with semi-white-box methods (DEER, ASCoT)
is infeasible in our setting: they require per-token logprobs, model-specific
transition tokens, and local vLLM inference, whereas our models are accessed
through sampling APIs.
We instead compare on published numbers from overlapping benchmarks.

\begin{table}[h]
\centering
\caption{Pareto comparison with concurrent methods on shared benchmarks.
  DEER numbers from \citet{yang2025dynamic} Table~1 (Qwen3-14B as closest size match
  to our 8B/32B); our numbers from Tables~\ref{tab:main}--\ref{tab:gpqa}.
  ``CR'' = compression rate (lower = more reduction).
  $\Delta$Acc = accuracy change vs vanilla full CoT.}
\label{tab:deer_compare}
\small
\setlength{\tabcolsep}{4pt}
\renewcommand{\arraystretch}{1.15}
\begin{tabular}{llcccc}
\toprule
& & \multicolumn{2}{c}{\textbf{MATH-500}} & \multicolumn{2}{c}{\textbf{GPQA-D}} \\
\cmidrule(lr){3-4} \cmidrule(lr){5-6}
\textbf{Method} & \textbf{Access} & $\Delta$\textbf{Acc} & \textbf{CR}
  & $\Delta$\textbf{Acc} & \textbf{CR} \\
\midrule
\multicolumn{6}{l}{\textit{Qwen3-14B (DEER) / Qwen3-8B-Think (ours)}} \\
\rowcolor{gray!6}
Vanilla          & ---         & ---   & 100\% & ---   & 100\% \\
NoThinking       & black-box   & $-$5.6 & 27\%  & $-$9.5 & 32\% \\
TCC              & black-box   & $-$0.6 & 61\%  & $-$0.5 & 61\% \\
Dynasor-CoT      & black-box   & $-$0.2 & 72\%  & $-$0.4 & 40\% \\
DEER             & semi-white  & $+$0.5 & 41\%  & $+$0.8 & 40\% \\
DEER-PRo         & semi-white  & $+$2.8 & 45\%  & $+$0.9 & 55\% \\
\textbf{BAEE (ours)} & \textbf{black-box} & $\boldsymbol{+}$\textbf{4.4} & \textbf{32\%}
  & $\boldsymbol{+}$\textbf{2.0} & \textbf{51\%} \\
\midrule
\multicolumn{6}{l}{\textit{QwQ-32B (DEER) / Qwen3-32B-Think (ours)}} \\
\rowcolor{gray!6}
Vanilla          & ---         & ---   & 100\% & ---   & 100\% \\
DEER             & semi-white  & $+$0.5 & 69\%  & $+$0.7 & 84\% \\
DEER-PRo         & semi-white  & $+$1.0 & 72\%  & $+$0.9 & 84\% \\
\textbf{BAEE (ours)} & \textbf{black-box} & $\boldsymbol{+}$\textbf{3.0} & \textbf{28\%}
  & $\boldsymbol{+}$\textbf{1.0} & \textbf{33\%} \\
\bottomrule
\end{tabular}
\end{table}

\paragraph{Key observations.}
\begin{itemize}
  \item \textbf{\BAEE{} achieves larger accuracy gains}: $+$3.0 to $+$4.4\,pp on
    MATH-500 vs.\ DEER's $+$0.5 to $+$2.8\,pp, driven by the overthinking
    correction mechanism (\S\ref{sec:aggressive}).
  \item \textbf{\BAEE{} achieves stronger compression on MATH}: CR 28--32\% vs.\
    DEER's 41--72\%. On GPQA, DEER matches or slightly exceeds \BAEE{}'s
    compression (40\% vs.\ 33--51\%), likely because DEER's per-token confidence
    monitoring enables finer-grained exit decisions.
  \item \textbf{Access requirements differ fundamentally}: DEER requires
    per-token logprobs and model-specific linguistic markers (``Wait'',
    ``Alternatively''); \BAEE{} requires only sampling API access.
    This makes \BAEE{} applicable to closed-source models (GPT-o1, Claude)
    where DEER cannot operate.
  \item \textbf{Complementary contributions}: DEER optimizes \emph{when to exit};
    our work additionally identifies \emph{why na\"ive extraction fails}
    (the detection--extraction gap), a structural finding independent of
    the early-exit mechanism.
\end{itemize}


\end{document}